\documentclass{article}
\usepackage{microtype}
\usepackage{graphicx}
\usepackage{subfigure}
\usepackage{booktabs}
\usepackage[accepted,nohyperref]{icml2021}

\usepackage{algorithm}
\usepackage{algorithmicx}
\usepackage[noend]{algpseudocode}
\usepackage{amsmath}
\usepackage{amssymb}
\usepackage{amsthm}
\usepackage{bbm}
\usepackage{bm}
\usepackage{color}
\usepackage{dirtytalk}
\usepackage{dsfont}
\usepackage{enumerate}
\usepackage{listings}
\usepackage{mathtools}
\usepackage{natbib}
\usepackage{times}
\usepackage{xspace}

\usepackage[usenames,dvipsnames]{xcolor}
\usepackage[bookmarks=false]{hyperref}
\hypersetup{
  pdftex,
  pdffitwindow=true,
  pdfstartview={FitH},
  pdfnewwindow=true,
  colorlinks,
  linktocpage=true,
  linkcolor=Green,
  urlcolor=Green,
  citecolor=Green
}
\usepackage[capitalize,noabbrev]{cleveref}

\usepackage[backgroundcolor=white]{todonotes}

\newcommand{\commentout}[1]{}
\newcommand{\junk}[1]{}

\Crefname{corollary}{Corollary}{Corollaries}
\Crefname{proposition}{Proposition}{Propositions}
\Crefname{theorem}{Theorem}{Theorems}
\Crefname{definition}{Definition}{Definitions}
\Crefname{assumption}{Assumption}{Assumptions}
\Crefname{example}{Example}{Examples}
\Crefname{remark}{Remark}{Remarks}
\Crefname{setting}{Setting}{Settings}
\Crefname{lemma}{Lemma}{Lemmas}

\usepackage{thmtools}
\declaretheorem[name=Theorem,refname={Theorem,Theorems},Refname={Theorem,Theorems}]{theorem}
\declaretheorem[name=Lemma,refname={Lemma,Lemmas},Refname={Lemma,Lemmas},sibling=theorem]{lemma}

\declaretheorem[name=Definition,refname={Definition,Definitions},Refname={Definition,Definitions},sibling=theorem]{definition}

\newcommand{\cA}{\mathcal{A}}
\newcommand{\cB}{\mathcal{B}}

\newcommand{\cH}{\mathcal{H}}

\newcommand{\cN}{\mathcal{N}}

\newcommand{\cP}{\mathcal{P}}

\newcommand{\cZ}{\mathcal{Z}}
\newcommand{\eps}{\varepsilon}

\newcommand{\realset}{\mathbb{R}}

\newcommand{\diag}[1]{\mathrm{diag}\left(#1\right)}

\newcommand{\E}[1]{\mathbb{E} \left[#1\right]}
\newcommand{\condE}[2]{\mathbb{E} \left[#1 \,\middle|\, #2\right]}

\newcommand{\prob}[1]{\mathbb{P} \left(#1\right)}
\newcommand{\condprob}[2]{\mathbb{P} \left(#1 \,\middle|\, #2\right)}

\newcommand{\abs}[1]{\left|#1\right|}

\newcommand*\dif{\mathop{}\!\mathrm{d}}
\newcommand{\floors}[1]{\left\lfloor#1\right\rfloor}
\newcommand{\I}[1]{\mathds{1} \! \left\{#1\right\}}

\newcommand{\maxnorm}[1]{\|#1\|_\infty}

\newcommand{\set}[1]{\left\{#1\right\}}

\newcommand{\T}{^\top}

\DeclareMathOperator*{\argmax}{arg\,max\,}

\mathchardef\mhyphen="2D

\newcommand{\linucb}{\ensuremath{\tt LinUCB}\xspace}

\newcommand{\metats}{\ensuremath{\tt MetaTS}\xspace}
\newcommand{\metatsd}{\ensuremath{\tt MetaTS/3}\xspace}
\newcommand{\metatsx}{\ensuremath{\tt MetaTSx3}\xspace}
\newcommand{\oraclets}{\ensuremath{\tt OracleTS}\xspace}

\newcommand{\ts}{\ensuremath{\tt TS}\xspace}

\icmltitlerunning{Meta-Thompson Sampling}

\begin{document}

\twocolumn[
\icmltitle{Meta-Thompson Sampling}

\icmlsetsymbol{equal}{*}

\begin{icmlauthorlist}
\icmlauthor{Branislav Kveton}{gr}
\icmlauthor{Mikhail Konobeev}{ua}
\icmlauthor{Manzil Zaheer}{gr}
\icmlauthor{Chih-wei Hsu}{gr}
\icmlauthor{Martin Mladenov}{gr}
\icmlauthor{Craig Boutilier}{gr}
\icmlauthor{Csaba Szepesv\'ari}{dm,ua}
\end{icmlauthorlist}

\icmlaffiliation{dm}{DeepMind}
\icmlaffiliation{gr}{Google Research}
\icmlaffiliation{ua}{University of Alberta}

\icmlcorrespondingauthor{Branislav Kveton}{bkveton@google.com}

\vskip 0.3in
]

\printAffiliationsAndNotice{}

\begin{abstract}
Efficient exploration in bandits is a fundamental online learning problem. We propose a variant of Thompson sampling that learns to explore better as it interacts with bandit instances drawn from an unknown prior. The algorithm meta-learns the prior and thus we call it \metats. We propose several efficient implementations of \metats and analyze it in Gaussian bandits. Our analysis shows the benefit of meta-learning and is of a broader interest, because we derive a novel prior-dependent Bayes regret bound for Thompson sampling. Our theory is complemented by empirical evaluation, which shows that \metats quickly adapts to the unknown prior.
\end{abstract}

\section{Introduction}
\label{sec:introduction}

A \emph{stochastic bandit} \citep{lai85asymptotically,auer02finitetime,lattimore19bandit} is an online learning problem where a \emph{learning agent} sequentially interacts with an environment over $n$ rounds. In each round, the agent pulls an \emph{arm} and then receives the arm's \emph{stochastic reward}. The agent aims to maximize its expected cumulative reward over $n$ rounds. It does not know the mean rewards of the arms \emph{a priori}, so must learn them by pulling the arms. This induces the well-known \emph{exploration-exploitation trade-off}: \emph{explore}, and learn more about an arm; or \emph{exploit}, and pull the arm with the highest estimated reward. In a clinical trial, the arm might be a treatment and its reward is the outcome of that treatment for a patient.

Bandit algorithms are typically designed to have low regret for some problem class of interest to the algorithm designer \citep{lattimore19bandit}. In practice, however, the problem class may not be perfectly specified at the time of the design. For instance, consider applying \emph{Thompson sampling (TS)} \citep{thompson33likelihood,chapelle11empirical,agrawal12analysis,russo18tutorial} to a $2$-armed bandit in which the prior distribution over mean arm rewards, a vital part of TS, is unknown. While the prior is unknown, the designer may know that it is one of two possible priors where either arm $1$ or arm $2$ is optimal with high probability. If the agent could learn which of the two priors has been realized, for instance by interacting repeatedly with bandit instances drawn from that prior, it could adapt its exploration strategy to the realized prior, and thereby incur much lower regret than would be possible without this adaptation.

We formalize this learning problem as follows. A learning agent sequentially interacts with $m$ bandit instances. Each interaction has $n$ rounds and we refer to it as a \emph{task}. The instances share a common structure, namely that their mean arm rewards are drawn from an unknown \emph{instance prior} $P_*$. While $P_*$ is not known, we assume that it is sampled from a \emph{meta-prior} $Q$, which the agent knows with certainty. The goal of the agent is to minimize the regret in each sampled instance almost as well as if it knew $P_*$. This is achieved by adapting to $P_*$ through interactions with the instances. This is a form of \emph{meta-learning} \citep{thrun96explanationbased,thrun98lifelong,baxter98theoretical,baxter00model}, where the agent learns to act from interactions with bandit instances.

We make the following contributions. First, we formalize the problem of Bayes regret minimization where the prior $P_*$ is unknown, and is learned by interactions with bandit instances sampled from it in $m$ tasks. Second, we propose \metats, a \emph{meta-Thompson sampling} algorithm that solves this problem. \metats maintains a distribution over the unknown $P_*$ in each task, which we call a \emph{meta-posterior} $Q_s$, and acts optimistically with respect to it. More specifically, in task $s$, it samples an estimate of $P_*$ as $P_s \sim Q_s$ and then runs TS with prior $P_s$ for $n$ rounds. We show how to implement \metats efficiently in Bernoulli and Gaussian bandits. In addition, we bound its Bayes regret in Gaussian bandits. Our analysis is conservative because it relies only on a single pull of each arm in each task. Nevertheless, it yields an improved regret bound due to adapting to $P_*$. The analysis is of broader interest, as we derive a novel \emph{prior-dependent} upper bound on the Bayes regret of TS. Our theoretical results are complemented by synthetic experiments, which show that \metats adapts quickly to the unknown prior $P_*$, and its regret is comparable to that of TS with a known $P_*$.

\section{Setting}
\label{sec:setting}

We start with introducing our notation. The set $\set{1, \dots, n}$ is denoted by $[n]$. The indicator $\I{E}$ denotes that event $E$ occurs. The $i$-th entry of vector $v$ is denoted by $v_i$. Sometimes we write $v(i)$ to avoid clutter. A diagonal matrix with entries $v$ is denoted by $\diag{v}$. We write $\tilde{O}$ for the big-O notation up to polylogarithmic factors.

Our setting is defined as follows. We have $K$ arms, where a bandit \emph{problem instance} is a vector of arm means $\theta \in \realset^K$. The agent sequentially interacts with $m$ bandit instances, which we index by $s \in [m]$. We refer to each interaction as a \emph{task}. At the beginning of task $s \in [m]$, an instance $\theta_{s, *}$ is sampled i.i.d.\ from an \emph{instance prior distribution} $P_*$. The agent interacts with $\theta_{s, *}$ for $n$ rounds. In round $t \in [n]$, it pulls one arm and observes a stochastic realization of its reward. We denote the pulled arm in round $t$ of task $s$ by $A_{s, t} \in [K]$, the stochastic rewards of all arms in round $t$ of task $s$ by $Y_{s, t} \in \realset^K$, and the reward of arm $i \in [K]$ by $Y_{s, t}(i)$. The result of the interactions in task $s$ is \emph{history}
\begin{align*}
  H_s
  = (A_{s, 1}, Y_{s, 1}(A_{s, 1}), \dots, A_{s, n}, Y_{s, n}(A_{s, n}))\,.
\end{align*}
We denote by $H_{1 : s} = H_1 \oplus \dots \oplus H_s$ a concatenated vector of the histories in tasks $1$ to $s$. We assume that the realized rewards $Y_{s, t}$ are i.i.d.\ with respect to both $s$ and $t$, and that their means are $\condE{Y_{s, t}}{\theta_{s, *} = \theta} = \theta$. For now, we need not assume that the reward noise is sub-Gaussian; but we do adopt this in our analysis (\cref{sec:analysis}).

The $n$-round \emph{Bayes regret} of a learning agent or algorithm over $m$ tasks with instance prior $P_*$ is
\begin{align*}
  R(m, n; P_*)
  = \sum_{s = 1}^m \condE{\sum_{t = 1}^n
  \theta_{s, *}(A_{s, *}) - \theta_{s, *}(A_{s, t})}{P_*}\,,
\end{align*}
where $A_{s, *} = \argmax_{i \in [K]} \theta_{s, *}(i)$ is the \emph{optimal arm} in the problem instance $\theta_{s, *}$ in task $s \in [m]$. The above expectation is over problem instances $\theta_{s, *}$ sampled from $P_*$, their realized rewards, and pulled arms.

We note that $R(1, n; P_*)$ is the standard definition of the $n$-round Bayes regret in a $K$-armed bandit \cite{russo14learning}, and that it is $\tilde{O}(\sqrt{K n})$ for Thomson sampling with prior $P_*$. Since all bandit instances $\theta_{s, *}$ are drawn i.i.d.\ from the same $P_*$, they provide no additional information about each other. Thus, the regret of TS with prior $P_*$ in $m$ such instances is $\tilde{O}(m \sqrt{K n})$. We validate this dependence empirically in \cref{sec:experiments}.

\begin{figure}[t]
  \centering
  \includegraphics[width=2.5in]{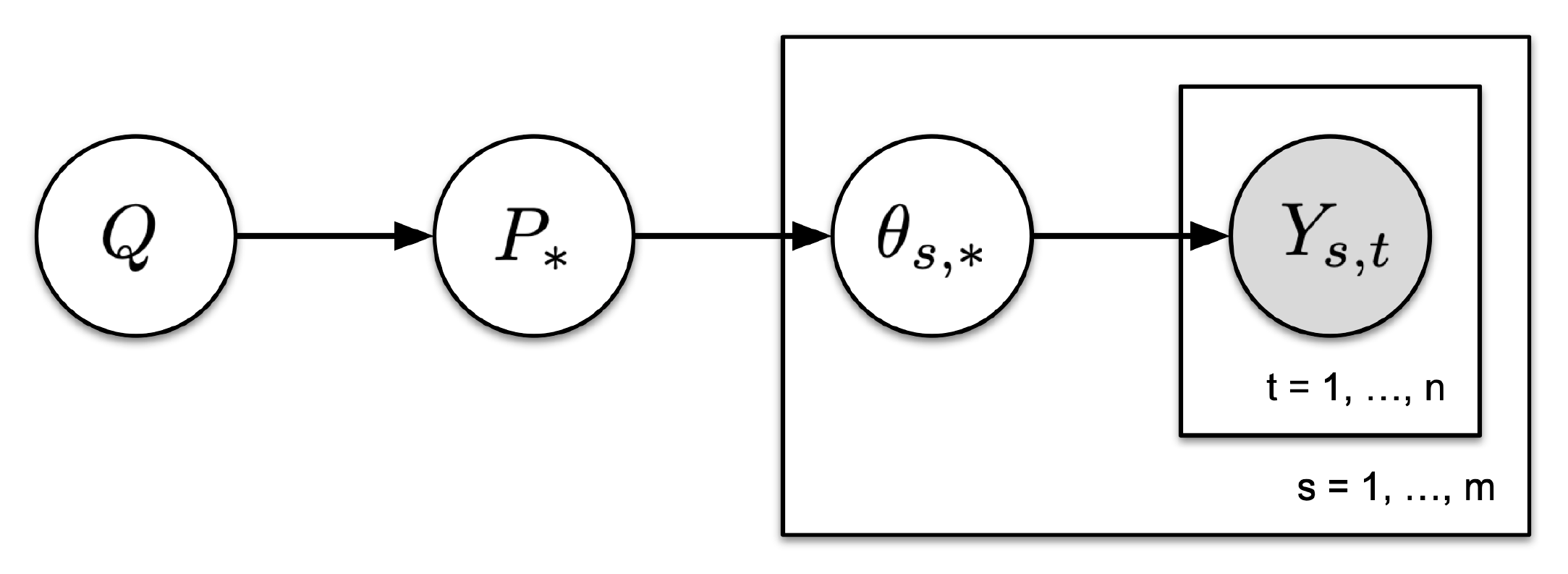}
  \caption{Graphical model of our bandit environment.}
  \label{fig:model}
\end{figure}

Note that Thompson sampling requires $P_*$ as an input. In this work, we try to attain the same regret \emph{without assuming that $P_*$ is known}. We formalize this problem in a Bayesian fashion. In particular, we assume the availability of a prior distribution $Q$ over problem instance priors, and that $P_* \sim Q$. We refer to $Q$ as a \emph{meta-prior} since it is a prior over priors. In Bayesian statistics, this would also be known as a hyper-prior \citep{gelman13bayesian}. The agent knows $Q$ but not $P_*$. We try to learn $P_*$ from sequential interactions with instances $\theta_{s, *}$, which are drawn i.i.d.\ from $P_*$ in each task. Note that $\theta_{s, *}$ is also unknown. The agent only observes its noisy realizations $Y_{s, t}$. We visualize relations of $Q$, $P_*$, $\theta_{s, *}$, and $Y_{s, t}$ in a graphical model in \cref{fig:model}.

One motivating example for using a meta-prior for exploration arises in recommender systems, in which exploration is used to assess the latent interests of users for different items, such as movies. In this case, each user can be treated as a bandit instance where the items are arms. A standard prior over user latent interests could readily be used by TS \citep{hong20latent}. However, in many cases, the algorithm designer may be uncertain about its true form. For instance, the designer may believe that most users have strong but noisy affinity for items in exactly one of several classes, but it is unclear which. Our work can be viewed as formalizing the problem of learning such a prior over user interests, which could be used to start exploring the preferences of \say{cold-start} users.

\section{Meta-Thompson Sampling}
\label{sec:meta-thompson sampling}

In this section, we present our approach to meta-learning in TS. We provide a general description in \cref{sec:algorithm meta-ts}. In \cref{sec:categorical meta-prior,sec:gaussian meta-prior}, we implement it in Bernoulli bandits with a categorical meta-prior and Gaussian bandits with a Gaussian meta-prior, respectively. In \cref{sec:measure-theoretic view}, we justify our approach beyond these specific instances.

\subsection{Algorithm \metats}
\label{sec:algorithm meta-ts}

Thompson sampling \citep{thompson33likelihood,chapelle11empirical,agrawal12analysis,russo18tutorial} is arguably the most popular and practical bandit algorithm. TS is parameterized by a prior, which is specified by the algorithm designer. In this work, we study a more general setting where the designer can model uncertainty over an unknown prior $P_*$ using a meta-prior $Q$. Our proposed algorithm meta-learns $P_*$ from sequential interactions with bandit instances drawn i.i.d.\ from $P_*$. Therefore, we call it \emph{meta-Thompson sampling} (\metats). In this subsection, we present \metats under the assumption that sample spaces are discrete. This eases exposition and guarantees that all conditional expectations are properly defined. We treat this topic more rigorously in \cref{sec:measure-theoretic view}.

\metats is a variant of TS that models uncertainty over the instance prior distribution $P_*$. This uncertainty is captured by a meta-posterior $Q_s$, a distribution over possible instance priors. We denote the \emph{meta-posterior} in task $s$ by $Q_s$, and assume that each $Q_s$ belongs to the same family as $Q$. By definition, $Q_1 = Q$ is the \emph{meta-prior}. \metats samples the \emph{instance prior distribution} $P_s$ in task $s$ from $Q_s$. Then it applies TS with the sampled prior $P_s$ to the bandit instance $\theta_{s, *}$ in task $s$ for $n$ rounds. Once the task is complete, it updates the meta-posterior in a standard Bayesian fashion
\begin{align}
  & Q_{s + 1}(P)
  \label{eq:posterior} \\
  & \propto \condprob{H_{1 : s}}{P_* = P} Q(P)
  \nonumber \\
  & = \condprob{H_s}{P_* = P} \prod_{\ell = 1}^{s - 1} \condprob{H_\ell}{P_* = P} Q(P)
  \nonumber \\
  & = \condprob{H_s}{P_* = P} Q_s(P)
  \nonumber \\
  & = \int_\theta \condprob{H_s}{\theta_{s, *} = \theta}
  \condprob{\theta_{s, *} = \theta}{P_* = P} \dif \theta \, Q_s(P)\,,
  \nonumber
\end{align}
where $\condprob{H_s}{P_* = P}$ and $\condprob{H_s}{\theta_{s, *} = \theta}$ are probabilities of observations in task $s$ given that the instance prior is $P$ and the problem instance is $\theta$, respectively. A rigorous justification of this update is given in \cref{sec:measure-theoretic view}. Specific instances of this update are in \cref{sec:categorical meta-prior,sec:gaussian meta-prior}.

The pseudocode for \metats is presented in \cref{alg:meta-ts}. The algorithm is simple, natural, and general; but has two potential shortcomings. First, it is unclear if it can be implemented efficiently. To address this, we develop efficient implementations for both Bernoulli and Gaussian bandits in \cref{sec:categorical meta-prior,sec:gaussian meta-prior}, respectively. Second, it is unclear whether \metats explores enough. Ideally, it should \emph{learn to} perform as well as TS with the true prior $P_*$. Intuitively, we expect this since the meta-posterior samples $P_s \sim Q_s$ should vary significantly in the direction of high variance in $Q_s$, which represents high uncertainty that can be reduced by exploring. We confirm this in \cref{sec:analysis}, where \metats is analyzed in Gaussian bandits.

\begin{algorithm}[t]
  \caption{\metats: Meta-learning Thompson sampling.}
  \label{alg:meta-ts}
  \begin{algorithmic}[1]
    \State \textbf{Inputs:} Meta-prior $Q$
    \Statex
    \State $Q_1 \gets Q$
    \For{$s = 1, \dots, m$}
      \State Sample $P_s \sim Q_s$
      \State Apply Thompson sampling with prior $P_s$ to
      \Statex \hspace{0.175in} problem instance $\theta_{s, *} \sim P_*$ for $n$ rounds
      \State Update meta-posterior $Q_{s + 1}$, as defined in \eqref{eq:posterior}
    \EndFor
  \end{algorithmic}
\end{algorithm}

\subsection{Bernoulli Bandit with a Categorical Meta-Prior}
\label{sec:categorical meta-prior}

Bernoulli Thompson sampling was the first instance of TS that was analyzed \citep{agrawal12analysis}. In this section, we apply \metats to this problem class.

We consider a Bernoulli bandit with $K$ arms that is parameterized by arm means $\theta \in [0, 1]^K$. The reward of arm $i$ in instance $\theta$ is drawn i.i.d.\ from $\mathrm{Ber}(\theta_i)$. To model uncertainty in the prior, we assume access to $L$ potential instance priors $\cP = \set{P^{(j)}}_{j = 1}^L$. Each prior $P^{(j)}$ is factored across the arms as
\begin{align*}
  P^{(j)}(\theta)
  & = \prod_{i = 1}^K \mathrm{Beta}(\theta_i; \alpha_{i, j}, \beta_{i, j}) \\
  & = \prod_{i = 1}^K \frac{\Gamma(\alpha_{i, j} + \beta_{i, j})}
  {\Gamma(\alpha_{i, j}) \Gamma(\beta_{i, j})}
  \theta_i^{\alpha_{i, j} - 1} (1 - \theta_i)^{\beta_{i, j} - 1}
\end{align*}
for some fixed $(\alpha_{i, j})_{i = 1}^K$ and $(\beta_{i, j})_{i = 1}^K$. The \emph{meta-prior} is a categorical distribution over $L$ classes of tasks. That is,
\begin{align*}
  Q(j)
  = \mathrm{Cat}(j; w)
  = w_j
\end{align*}
for $w \in \Delta_{L - 1}$, where $w$ is a vector of initial beliefs into each instance prior and $\Delta_{L - 1}$ is the $L$-dimensional simplex. The tasks are generated as follows. First, the instance prior is set as $P_* = P^{(j_*)}$ where $j_* \sim Q$. Then, in each task $s$, a Bernoulli bandit instance is sampled as $\theta_{s, *} \sim P_*$.

\metats is implemented as follows. The meta-posterior in task $s$ is
\begin{align*}
  Q_s(j)
  = \mathrm{Cat}(j; \hat{w}_s)
  = \hat{w}_{s, j}\,,
\end{align*}
where $\hat{w}_s \in \Delta_{L - 1}$ is a vector of posterior beliefs into each instance prior. The instance prior in task $s$ is $P_s = P^{(j_s)}$ where $j_s \sim Q_s$. After interacting with bandit instance $\theta_{s, *}$, the meta-posterior is updated using $Q_{s + 1}(j) \propto f(j) \, Q_s(j)$, where
\begin{align*}
  f(j)
  = {} & \int_\theta \condprob{H_s}{\theta_{s, *} = \theta}
  \condprob{\theta_{s, *} = \theta}{j_* = j} \dif \theta \\
  = {} & \prod_{i = 1}^K \frac{\Gamma(\alpha_{i, j} + \beta_{i, j})}
  {\Gamma(\alpha_{i, j}) \Gamma(\beta_{i, j})} \times {} \\
  & \int_{\theta_i} \theta_i^{\alpha_{i, j} + N_{i, s}^+ - 1}
  (1 - \theta_i)^{\beta_{i, j} + N_{i, s}^- - 1} \dif \theta_i \\
  = {} & \prod_{i = 1}^K \frac{\Gamma(\alpha_{i, j} + \beta_{i, j})
  \Gamma(\alpha_{i, j} + N_{i, s}^+)
  \Gamma(\beta_{i, j} + N_{i, s}^-)}
  {\Gamma(\alpha_{i, j}) \Gamma(\beta_{i, j})
  \Gamma(\alpha_{i, j} + \beta_{i, j} + T_{i, s})}\,.
\end{align*}
Here $\cA_{i, s} = \set{t \in [n]: A_{s, t} = i}$ is the set of rounds where arm $i$ is pulled in task $s$ and $T_{i, s} = \abs{\cA_{i, s}}$ is the number of these rounds. In addition, $N_{i, s}^+ = \sum_{t \in \cA_{i, s}} Y_{s, t}(i)$ denotes the number of positive observations of arm $i$ and $N_{i, s}^- = T_{i, s} - N_{i, s}^+$ is the number of its negative observations.

The above derivation can be generalized in a straightforward fashion to any categorical meta-prior whose instance priors $P^{(j)}$ lie in some exponential family.

\subsection{Gaussian Bandit with a Gaussian Meta-Prior}
\label{sec:gaussian meta-prior}

Gaussian distributions have many properties that allow for tractable analysis, such as that the posterior variance is independent of observations, which we exploit in \cref{sec:analysis}. In this section, we present a computationally-efficient implementation for this problem class.

We consider a Gaussian bandit with $K$ arms that is parameterized by arm means $\theta \in \realset^K$. The reward of arm $i$ in instance $\theta$ is drawn i.i.d.\ from $\cN(\theta_i, \sigma^2)$. We have a continuum of instance priors, parameterized by a vector of means $\mu \in \realset^K$ and defined as $P(\theta) = \cN(\theta; \mu, \sigma_0^2 I_K)$. The noise $\sigma_0$ is fixed. The \emph{meta-prior} is a Gaussian distribution over instance prior means $Q(\mu) = \cN(\mu; \mathbf{0}, \sigma_q^2 I_K)$, where $\sigma_q$ is assumed to be known. The tasks are generated as follows. First, the instance prior is set as $P_* = \cN(\mu_*, \sigma_0^2 I_K)$ where $\mu_* \sim Q$. Then, in each task $s$, a Gaussian bandit instance is sampled as $\theta_{s, *} \sim P_*$.

\metats is implemented as follows. The meta-posterior in task $s$ is
\begin{align*}
  Q_s(\mu)
  = \cN(\mu; \hat{\mu}_s, \hat{\Sigma}_s)\,,
\end{align*}
where $\hat{\mu}_s \in \realset^K$ is an estimate of $\mu_*$ and $\hat{\Sigma}_s \in \realset^{K \times K}$ is a diagonal covariance matrix. The instance prior in task $s$ is $P_s(\theta) = \cN(\theta; \tilde{\mu}_s, \sigma_0^2 I_K)$ where $\tilde{\mu}_s \sim Q_s$. After interacting with bandit instance $\theta_{s, *}$, the meta-posterior is updated as $Q_{s + 1}(\mu) \propto f(\mu) \, Q_s(\mu)$, where
\begin{align*}
  & f(\mu) \\
  & = \int_\theta \condprob{H_s}{\theta_{s, *} = \theta}
  \condprob{\theta_{s, *} = \theta}{\mu_* = \mu} \dif \theta \\
  & = \prod_{i = 1}^K \int_{\theta_i}
  \left[\prod_{t \in \cA_{i, s}} \cN(Y_{s, t}(i); \theta_i, \sigma^2)\right]
  \cN(\theta_i; \mu_i, \sigma_0^2) \dif \theta_i\,.
\end{align*}
Here $\cA_{i, s} = \set{t \in [n]: A_{s, t} = i}$ is the set of rounds where arm $i$ is pulled in task $s$ and $T_{i, s} = \abs{\cA_{i, s}}$ is the number of such rounds, as in \cref{sec:categorical meta-prior}.

Since $\hat{\Sigma}_s = \diag{\hat{\sigma}_s^2}$ is a diagonal covariance matrix, it is fully characterized by a vector of individual arm variances $\hat{\sigma}_s^2 \in \realset^K$. The parameters $\hat{\mu}_s$ and $\hat{\sigma}_s^2$ are updated, based on \cref{lem:gaussian posterior update} in \cref{sec:technical lemmas}, as
\begin{align*}
  \hat{\mu}_{s + 1, i}
  & = \hat{\sigma}_{s + 1, i}^2
  \left(\frac{\hat{\mu}_{s, i}}{\hat{\sigma}_{s, i}^2} +
  \frac{T_{i, s}}{T_{i, s} \sigma_0^2 + \sigma^2}
  \frac{\sum_{t \in \cA_{i, s}} Y_{s, t}(i)}{T_{i, s}}\right)\,, \\
  \hat{\sigma}_{s + 1, i}^{-2}
  & = \hat{\sigma}_{s, i}^{-2} + \frac{T_{i, s}}{T_{i, s} \sigma_0^2 + \sigma^2}\,.
\end{align*}
This update can be also derived using \eqref{eq:linear bandit posterior} in \cref{sec:bayesian multi-task regression}, when all covariance matrices are assumed to be diagonal.

The above update has a very nice interpretation. The posterior mean $\hat{\mu}_{s + 1, i}$ of arm $i$ is a weighted sum of the mean reward estimate of arm $i$ in task $s$ and the earlier posterior mean $\hat{\mu}_{s, i}$. The weight of the estimate depends on how good it is. Specifically, it varies from $1 / (\sigma_{0, i}^2 + \sigma^2)$, when arm $i$ is pulled only once, to $1 / \sigma_{0, i}^2$, when $T_{i, s} \to \infty$. This is the minimum amount of uncertainty that cannot be reduced by more pulls, due to the fact that $\theta_{s, *}$ is a single observation of the unknown $\mu_*$ with covariance $\sigma_0^2 I_K$.

\subsection{Measure-Theoretic View and the General Case}
\label{sec:measure-theoretic view}

We now present a more general measure-theoretic specification of our meta-bandit setting. Let $\cZ$ be the set of outcomes for the hidden variable $Z$ that is sampled from a meta-prior and $\sigma(\cZ)$ be the $\sigma$-algebra over this set. Similarly, let $\Theta$ be the set of possible bandit environments $\theta \in \Theta$ and $\sigma(\Theta)$ be the $\sigma$-algebra over this set. While in this work we focus on environments characterized only by their mean reward vectors, this parameterization could be more general, and for example include the variance of mean reward vectors. The formal definition of a $K$-armed Bayesian meta-bandit is as follows.

\begin{definition}
A $K$-armed Bayesian meta-bandit is a tuple $\cB=(\cZ, \sigma(\cZ), Q, \Theta, \sigma(\Theta), P, \rho)$, where $(\cZ,\sigma(\cZ))$ is a measurable space; the \emph{meta-prior} $Q$ is a probability measure over $(\cZ,\sigma(\cZ))$; the \emph{prior} $P$ is a probability kernel from $(\cZ,\sigma(\cZ))$ to $(\Theta,\sigma(\Theta))$; and $\rho=(\rho_{\theta,i}: \theta\in\Theta,i\in[K])$ is a probability kernel from $\Theta\times[K]$ to $(\realset,\mathfrak{B}(\realset))$, where $\mathfrak{B}(\realset)$ is the Borel $\sigma$-algebra of $\realset$ and $\rho_{\theta,i}$ is the reward distribution associated with arm $i$ in bandit $\theta$.
\end{definition}

We use lowercase letters to denote realizations of random variables. Let $P_z$ be a distribution of bandit instances under $Z = z$. We assume that a new environment $\theta \in \Theta$ is sampled from the same measure $P_z$ at the beginning of each task with the same realization of hidden variable $Z$ sampled from the meta-prior $Q$ beforehand. 

A bandit algorithm consists of kernels $\pi_{s,t}$ that take as input a history of interactions consisting of the pulled arms and observed rewards up to round $t$ in task $s$, and output a probability measure over the arms. A bandit algorithm is connected with a Bayesian meta-bandit environment $\cB$ to produce a sequence of chosen arms and observed rewards. Formally, let $\Omega_{s,t}= ([K]\times\realset)^{(s-1)n+t-1}\subset\realset^{2((s-1)n+t-1)}$  for each $t\in[n]$ and $s\in[m]$. Then a bandit algorithm or policy is a tuple $\pi=(\pi_{s,t})_{s,t=1}^{m,n}$ such that each $\pi_{s,t}$ is a kernel from $(\Omega_{s,t},\mathfrak{B}(\realset^{2((s-1)n+t-1)}))$ to $([K], 2^{[K]})$ that interacts with a Bayesian meta-bandit $\cB$ over $m$ tasks, each lasting $n$ rounds and producing a sequence of random variables
\begin{align*}
  & A_{1,1},X_{1,1},\dots,A_{1,n},X_{1,n}\,, \dots, \\
  & A_{m,1},X_{m,1}\dots,A_{m,n},X_{m,n}\,,
\end{align*}
where $X_{s,t}=Y_{s,t}(A_{s,t})$ is the reward in round $t$ of task $s$. The probability measure over these variables, $\mathbb{P}_{z,\theta_1,\dots,\theta_m,\pi}$, is guaranteed to exist by the Ionescu-Tulcea theorem \citep{tulcea1949mesures}. Furthermore, the conditional probabilities of transitions of this measure are equal to the kernels
\begin{align*}
&\condprob{\theta_{s}\in\cdot}{
z,\theta_1,\dots,\theta_{s-1},H_{1:s-1}}
= P_z(\theta_{s}\in\cdot)\,, \\
&\condprob{A_{s,t}\in\cdot}{
z,\theta_1,\dots,\theta_s,H_{1:s-1},H_{s,t}} = \\
&\qquad\pi_{s,t}(
A_{s,t}\in\cdot\,|
\,H_{1:s-1},H_{s,t})\,, \\
&\condprob{X_{s,t}\in\cdot}{
z,\theta_1,\dots,\theta_s,
H_{1:s-1},H_{s,t},A_{s,t}} = \\
&\qquad\rho_{\theta_s,A_{s,t}}(X_{s,t}\in\cdot)\,,
\end{align*}
where $H_{s,t}= (X_{s,\ell},A_{s,\ell})_{\ell=1}^{t-1}$. The following lemma says that both the task-posterior $P_z(\cdot|h_{s,t})$ for any $z\in\cZ$ and the meta-posterior $Q(\cdot|h_{1:s-1})$ depend only on the pulled arms according to $\pi$, but not the exact form of $\pi$. 

\begin{restatable}[]{lemma}{measuretheorylemma}
Assume that there exists a $\sigma$-finite measure $\lambda_\rho$ on $(\realset,\mathfrak{B}(\realset))$ such that $\rho_{\theta,i}$ is absolutely continuous with respect to $\lambda_\rho$ for all $\theta\in\Theta$ and $i\in[K]$. Then the task-posterior and meta-posterior exist and have the following form
\begin{align*}
&P_z(S_1|h_{s,t})=
\frac{\int_{S_1} 
\prod_{j=1}^{t-1} p_{\theta_s,a_{s,j}}(x_{s,j})
\dif P_{z}(\theta_s)
}{\int_\Theta 
\prod_{j=1}^{t-1} p_{\theta_s,a_{s,j}}(x_{s,j})
\dif P_{z}(\theta_s)}, \\
&Q(S_2|h_{1:s-1}) = \\
&\quad\frac{
\int_{S_2} 
\left[\prod_{\ell=1}^{s-1} \prod_{j=1}^n 
\int_{\Theta}p_{
\theta,a_{\ell,j}}(x_{\ell,j})
\dif P_z(\theta)
\right] \dif Q(z)}{
\int_\cZ \left[\prod_{\ell=1}^{s-1} \prod_{j=1}^n
\int_{\Theta}p_{
\theta,a_{\ell,j}}(x_{\ell,j})
\dif P_z(\theta)
\right]\dif Q(z)},
\end{align*}
for any $S_1\in\sigma(\Theta),z\in\cZ,S_2\in\sigma(\cZ)$.
\end{restatable}

The proof is provided in \cref{sec:measure-theoretic-view}. \metats samples $Z_s$ from the meta-posterior $Q_s=Q(\cdot|h_{1:s-1})$ at the beginning of each task $s\in[m]$ and then pulls arms according to the samples from $P_{Z_s}(\cdot|h_{s,t})$. The above lemma shows that to compute the posteriors we only need the distributions of rewards and integrate them over the environments (for the task-posterior) or over both the environments and the hidden variables (for the meta-posterior).These integrals can be derived analytically, for example, in the case of conjugate priors \cref{sec:gaussian meta-prior,sec:categorical meta-prior}.

\section{Analysis}
\label{sec:analysis}

We bound the Bayes regret of \metats in Gaussian bandits (\cref{sec:gaussian meta-prior}). This section is organized as follows. We state the bound and sketch its proof in \cref{sec:regret bound}, and discuss it in \cref{sec:discussion}. In \cref{sec:key lemmas}, we present the key lemmas. Finally, in \cref{sec:beyond gaussian bandits}, we discuss how our analysis can be extended beyond Gaussian bandits.

\subsection{Regret Bound}
\label{sec:regret bound}

We analyze \metats under the assumption that each arm is pulled at least once per task. Although this suffices to show benefits of meta-learning, it is conservative. A less conservative analysis would require understanding how Thompson sampling with a misspecified prior pulls arms. In particular, we would require a high-probability lower bound on the number of pulls of each arm. To the best of our knowledge, such as a bound does not exist and is non-trivial to derive. To guarantee that each arm is pulled at least once, we pull each arm in the last $K$ rounds of each task. This is to avoid any interference with our posterior sampling analyses in the earlier rounds. Other than this, \metats is analyzed exactly as described in \cref{sec:algorithm meta-ts,sec:gaussian meta-prior}.

Recall the following definitions in our setting (\cref{sec:gaussian meta-prior}). The meta-prior is $Q = \cN(\mathbf{0}, \sigma_q^2 I_K)$. The instance prior is $P_* = \cN(\mu_*, \sigma_0^2 I_K)$, where $\mu_* \sim Q$ is chosen before the learning agent interact with the tasks. Then, in each task $s$, a problem instance is drawn i.i.d.\ as $\theta_{s, *} \sim P_*$. Our main result is the following Bayes regret bound.

\begin{theorem}
\label{thm:meta-ts regret bound} The Bayes regret of \metats over $m$ tasks with $n$ rounds each is
\begin{align*}
  & R(m, n; P_*) \leq \\
  & \quad c_1 \sqrt{K} \left(\sqrt{n + \sigma^2 \sigma_0^{-2} K} -
  \sqrt{\sigma^2 \sigma_0^{-2} K}\right) m + {} \\
  & \quad c_2(\delta) c_3(\delta) K n^2 \sqrt{m} +
  \tilde{O}(K m + n)
\end{align*}
with probability at least $1 - (2 m + 1) \delta$, where
\begin{align*}
  c_1
  & = 4 \sqrt{2 \sigma^2 \log n}\,, \\
  c_2(\delta)
  & = 2 \left(\sqrt{2 \sigma_q^2 \log(2 K / \delta)} +
  \sqrt{2 \sigma_0^2 \log n}\right)\,, \\
  c_3(\delta)
  & = 8 \sqrt{(\sigma_0^2 + \sigma^2) \log(4 K / \delta) /
  (\pi \sigma_0^2)}\,.
\end{align*}
The probability is over realizations of $\mu_*$, $\theta_{s, *}$, and $\hat{\mu}_s$.
\end{theorem}
\begin{proof}
First, we bound the magnitude of $\mu_*$. Specifically, since $\mu_* \sim \cN(\mathbf{0}, \sigma_q^2 I_K)$, we have that
\begin{align}
  \maxnorm{\mu_*}
  \leq \sqrt{2 \sigma_q^2 \log(2 K / \delta)}
  \label{eq:prior scale}
\end{align}
holds with probability at least $1 - \delta$.

Now we fix task $s \geq 2$ and decompose its regret. Let $A_{s, *}$ be the optimal arm in instance $\theta_{s, *}$, $A_{s, t}$ be the pulled arm in round $t$ by TS with misspecified prior $P_s$, and $\tilde{A}_{s, t}$ be the pulled arm in round $t$ by TS with correct prior $P_*$. Then
\begin{align*}
  \condE{\sum_{t = 1}^n \theta_{s, *}(A_{s, *}) - \theta_{s, *}(A_{s, t})}{P_*}
  = R_{s, 1} + R_{s, 2}\,,
\end{align*}
where
\begin{align*}
  R_{s, 1}
  & = \condE{\sum_{t = 1}^n \theta_{s, *}(A_{s, *}) -
  \theta_{s, *}(\tilde{A}_{s, t})}{P_*}\,, \\
  R_{s, 2}
  & = \condE{\sum_{t = 1}^n \theta_{s, *}(\tilde{A}_{s, t}) -
  \theta_{s, *}(A_{s, t})}{P_*}\,.
\end{align*}
The term $R_{s, 1}$ is the regret of hypothetical TS that knows $P_*$. This TS is introduced only for the purpose of analysis and is the optimal policy. The term $R_{s, 2}$ is the difference in the expected $n$-round rewards of TS with priors $P_s$ and $P_*$, and vanishes as the number of tasks $s$ increases.

To bound $R_{s, 1}$, we apply \cref{lem:prior-dependent regret bound} with $\delta = 1 / n$ and get
\begin{align*}
  R_{s, 1}
  \leq {} & 4 \sqrt{2 \sigma^2 K \log n} \times {} \\
  & \left(\sqrt{n + \sigma^2 \sigma_0^{-2} K} - \sqrt{\sigma^2 \sigma_0^{-2} K}\right) +
  \tilde{O}(K)\,,
\end{align*}
where $\tilde{O}(K)$ corresponds to the $c(\delta)$ term in \cref{lem:prior-dependent regret bound}. To bound $R_{s, 2}$, we apply \cref{lem:misspecified prior} with $\delta = 1 / n$ and get
\begin{align*}
  R_{s, 2}
  \leq {} & 2 \left(\maxnorm{\mu_*} + \sqrt{2 \sigma_0^2 \log n}\right)
  \sqrt{\frac{2}{\pi \sigma_0^2}} \times {} \\
  & K n^2 \maxnorm{\tilde{\mu}_s - \mu_*} +
  \tilde{O}(K)\,,
\end{align*}
where $\tilde{O}(K)$ is the first term in \cref{lem:misspecified prior}, after we bound $\maxnorm{\mu_*}$ in it using \eqref{eq:prior scale}.

Now we sum up our bounds on $R_{s, 1} + R_{s, 2}$ over all tasks $s \geq 2$ and get
\begin{align*}
  & c_1 \sqrt{K} \left(\sqrt{n + \sigma^2 \sigma_0^{-2} K} -
  \sqrt{\sigma^2 \sigma_0^{-2} K}\right) m + {} \\
  & c_2(\delta) \sqrt{\frac{2}{\pi \sigma_0^2}} K n^2
  \sum_{s = 2}^m \maxnorm{\tilde{\mu}_s - \mu_*} +
  \tilde{O}(K m)\,,
\end{align*}
where $c_1$ and $c_2(\delta)$ are defined in the main claim. Then we apply \cref{lem:meta-posterior concentration} to each term $\maxnorm{\tilde{\mu}_s - \mu_*}$ and have with probability at least $1 - m \delta$ that
\begin{align*}
  \sum_{s = 2}^m \maxnorm{\tilde{\mu}_s - \mu_*}
  \leq 4 \sqrt{2 (\sigma_0^2 + \sigma^2) m \log(4 K / \delta)}\,,
\end{align*}
where $\sqrt{m}$ arises from summing up the $O(1 / \sqrt{s})$ terms in \cref{lem:meta-posterior concentration}, using \cref{lem:reciprocal root sum} in \cref{sec:technical lemmas}. This concludes the main part of the proof.

We finish with an upper bound on the regret in task $1$ and the cost of pulling each arm once at the end of each task. This is can done as follows. Since $\theta_{s, *} \sim \cN(\mu_*, \sigma_0^2 I_K)$,
\begin{align}
  \maxnorm{\theta_{s, *} - \mu_*}
  \leq \sqrt{2 \sigma_0^2 \log(2 K / \delta)}
  \label{eq:instance scale}
\end{align}
holds with probability at least $1 - \delta$ in any task $s$. From \eqref{eq:prior scale} and \eqref{eq:instance scale}, we have with a high probability that
\begin{align*}
  \maxnorm{\theta_{s, *}}
  & \leq \maxnorm{\theta_{s, *} - \mu_*} + \maxnorm{\mu_*} \\
  & \leq 2 \sqrt{(\sigma_q^2 +\sigma_0^2) \log(2 K / \delta)}\,.
\end{align*}
This yields a high-probability upper bound of
\begin{align*}
  4 \sqrt{(\sigma_q^2 +\sigma_0^2) \log(2 K / \delta)} (n + K m)
\end{align*}
on the regret in task $1$ and pulling each arm once at the end of all tasks. This concludes our proof.
\end{proof}

\subsection{Discussion}
\label{sec:discussion}

If we assume a \say{large $m$ and $n$} regime, where the learning agent improves with more tasks but also needs to perform well in each task, the most important terms in \cref{thm:meta-ts regret bound} are those where $m$ and $n$ interact. Using these terms, our bound can be summarized as
\begin{align}
  \!\!\!\!\!\!
  \sqrt{K} \left[\sqrt{n + \sigma^2 \sigma_0^{-2} K} -
  \sqrt{\sigma^2 \sigma_0^{-2} K}\right] m +
  K n^2 \sqrt{m}
  \label{eq:regret bound}
\end{align}
and holds with probability $1 - (2 m + 1) \delta$ for any $\delta > 0$.

Our bound in \eqref{eq:regret bound} can be viewed as follows. The first term is the regret of Thompson sampling with the correct prior $P_*$. It is linear in the number of tasks $m$, since \metats solves $m$ exploration problems. The second term captures the cost of learning $P_*$. Since it is sublinear in the number of tasks $m$, \metats is near optimal in the regime of \say{large $m$}.

We compare our bound to two baselines. The first baseline is TS with a known prior $P_*$. The regret of this TS can be bounded using \cref{lem:prior-dependent regret bound} and includes only the first term in \eqref{eq:regret bound}. In the regime of \say{large m}, this term dominates the regret of \metats, and thus \metats is near optimal.

The second baseline is \emph{agnostic} Thompson sampling, which does use the structure $\theta_{s, *} \sim P_* \sim Q$. Instead, it marginalizes out $Q$. In our setting, this can be equivalently viewed as assuming $\theta_{s, *} \sim \cN(\mathbf{0}, (\sigma_q^2 + \sigma_0^2) I_K)$. For this prior, the Bayes regret is $\E{R(m, n; P_*)}$, where the expectation is over $P_* \sim Q$. Again, we can apply \cref{lem:prior-dependent regret bound} and show that $\E{R(m, n; P_*)}$ has an upper bound of
\begin{align*}
  \sqrt{K} \left(\sqrt{n + \sigma^2 \tilde{\sigma}^{-1} K} -
  \sqrt{\sigma^2 \tilde{\sigma}^{-1} K}\right) m\,,
\end{align*}
where $\tilde{\sigma}^2 = \sigma_q^2 + \sigma_0^2$. Clearly, $\tilde{\sigma}^2 > \sigma_0^2$; and therefore the difference of the above square roots is always larger than in \eqref{eq:regret bound}. So, in the regime of \say{large m}, \metats has a lower regret than this baseline.

\subsection{Key Lemmas}
\label{sec:key lemmas}

Now we present the three key lemmas used in the proof of \cref{thm:meta-ts regret bound}. They are proved in \cref{sec:regret bound lemmas}.

The first lemma is a prior-dependent upper bound on the Bayes regret of TS.

\begin{restatable}[]{lemma}{priordependentregretbound}
\label{lem:prior-dependent regret bound} Let $\theta_*$ be arm means in a $K$-armed Gaussian bandit that are generated as $\theta_* \sim P_* = \cN(\mu_*, \sigma_0^2 I_K)$. Let $A_*$ be the optimal arm under $\theta_*$ and $A_t$ be the pulled arm in round $t$ by TS with prior $P_*$. Then for any $\delta > 0$,
\begin{align*}
  & \E{\sum_{t = 1}^n \theta_*(A_*) - \theta_*(A_t)}
  \leq c(\delta) + {} \\
  & \quad 4 \sqrt{2 \sigma^2 K \log(1 / \delta)}
  \left(\sqrt{n + \sigma^2 \sigma_0^{-2} K} -
  \sqrt{\sigma^2 \sigma_0^{-2} K}\right)\,,
\end{align*}
where $c(\delta) = 2 \sqrt{2 \sigma_0^2 \log(1 / \delta)} K + \sqrt{2 \sigma_0^2 / \pi} K n \delta$.
\end{restatable}

The effect of the prior is reflected in the difference of the square roots. As the prior width narrows and $\sigma_0 \to 0$, the difference decreases, which shows that a more concentrated prior leads to less exploration. The algebraic form of the bound is also expected. Roughly speaking, $\sqrt{\sigma^2 \sigma_0^{-2} K}$ is the sum of confidence interval widths in the Bayes regret analysis that cannot occur, because the prior width is $\sigma_0$.

The bound in \cref{lem:prior-dependent regret bound} differs from other similar bounds in the literature \cite{lu19informationtheoretic}. One difference is that the Cauchy-Schwarz inequality is not used in its proof. Therefore, $\sigma_0$ is in the square root instead of the logarithm. The resulting bound is tighter for $\sigma^2 \sigma_0^{-2} K \ll n$. Another difference is that information-theory arguments are not used in the proof. The dependence on $\sigma_0$ is a result of carefully characterizing the posterior variance of $\theta_*$ in each round. Our proof is simple and easy to follow.

The second lemma bounds the difference in the expected $n$-round rewards of TS with different priors.

\begin{restatable}[]{lemma}{misspecifiedprior}
\label{lem:misspecified prior} Let $\theta_*$ be arm means in a $K$-armed Gaussian bandit that are generated as $\theta_* \sim P_* = \cN(\mu_*, \sigma_0^2 I_K)$. Let $\cN(\hat{\mu}, \sigma_0^2 I_K)$ and $\cN(\tilde{\mu}, \sigma_0^2 I_K)$ be two TS priors such that $\maxnorm{\hat{\mu} - \tilde{\mu}} \leq \varepsilon$. Let $\hat{\theta}_t$ and $\tilde{\theta}_t$ be their respective posterior samples in round $t$, $\hat{A}_t$ and $\tilde{A}_t$ be the pulled arms under these samples. Then for any $\delta > 0$,
\begin{align*}
  & \E{\sum_{t = 1}^n \theta_*(\hat{A}_t) - \theta_*(\tilde{A}_t)}
  \leq 4 \left(\sqrt{\frac{\sigma_0^2}{2 \pi}} + \maxnorm{\mu_*}\right)
  K n \delta + {} \\
  & \quad 2 \left(\maxnorm{\mu_*} + \sqrt{2 \sigma_0^2 \log(1 / \delta)}\right)
  \sqrt{\frac{2}{\pi \sigma_0^2}} K n^2 \eps\,.
\end{align*}
\end{restatable}

The key dependence in \cref{lem:misspecified prior} is that the bound is linear in the difference of prior means $\eps$. The bound is also $O(n^2)$. Although this is unfortunate, it cannot be improved in general if we want to keep linear dependence on $\eps$. The $O(n^2)$ dependence arises in the proof as follows. We bound the difference in the expected $n$-round rewards of TS with two different priors by the probability that the two TS instances deviate in each round multiplied by the maximum reward that can be earned from that round. The probability is $O(\eps)$ and the maximum reward is $O(n)$. This bound is applied $n$ times, in each round, and thus the $O(n^2 \eps)$ dependence.

The last lemma shows the concentration of meta-posterior sample means.

\begin{restatable}[]{lemma}{metaposteriorconcentration}
\label{lem:meta-posterior concentration} Let $\mu_* \sim \cN(\mathbf{0}, \sigma_q^2 I_K)$ and the prior parameters in task $s$ be sampled as $\tilde{\mu}_s \mid H_{1 : s - 1} \sim \cN(\hat{\mu}_s, \hat{\Sigma}_s)$. Then
\begin{align*}
  \maxnorm{\tilde{\mu}_s - \mu_*}
  \leq 2 \sqrt{2 \frac{\sigma_0^2 + \sigma^2}{(\sigma_0^2 + \sigma^2) \sigma_q^{-2} + s - 1}
  \log(4 K / \delta)}
\end{align*}
holds jointly over all tasks $s \in [m]$ with probability at least $1 - m \delta$.
\end{restatable}

The key dependence is that the bound is $O(1 / \sqrt{s})$ in task $s$, which provides an upper bound on $\eps$ in \cref{lem:misspecified prior}. After we sum up these upper bounds over all $s \in [m]$ tasks, we get the $O(\sqrt{m})$ term in \cref{thm:meta-ts regret bound}.

\subsection{Beyond Gaussian Bandits}
\label{sec:beyond gaussian bandits}

We analyze Gaussian bandits with a known prior covariance matrix (\cref{sec:regret bound}) because this simplifies algebra and is easy to interpret. We believe that a similar analysis can be conducted for other bandit problems based on the following high-level interpretation of our key lemmas (\cref{sec:key lemmas}).

\cref{lem:prior-dependent regret bound} says that more a concentrated prior in TS yields lower regret. This is expected in general, as less uncertainty about the problem instance leads to lower regret.

\cref{lem:misspecified prior} says that the difference in the expected $n$-round rewards of TS with different priors can be bounded by the difference of the prior parameters. This is expected for any prior that is smooth in its parameters.

\cref{lem:meta-posterior concentration} says that the meta-posterior concentrates as the number of tasks increases. When each arm is pulled at least once per task, as we assume in \cref{sec:regret bound}, \metats gets at least one noisy observation of the prior per task, and any exponential-family meta-posterior would concentrate.

\section{Experiments}
\label{sec:experiments}

We experiment with three problems. In each problem, we have $m = 20$ tasks with a horizon of $n = 200$ rounds. All results are averaged over $100$ runs, where $P_* \sim Q$ in each run.

\begin{figure*}[t]
  \centering
  \includegraphics[width=6.6in]{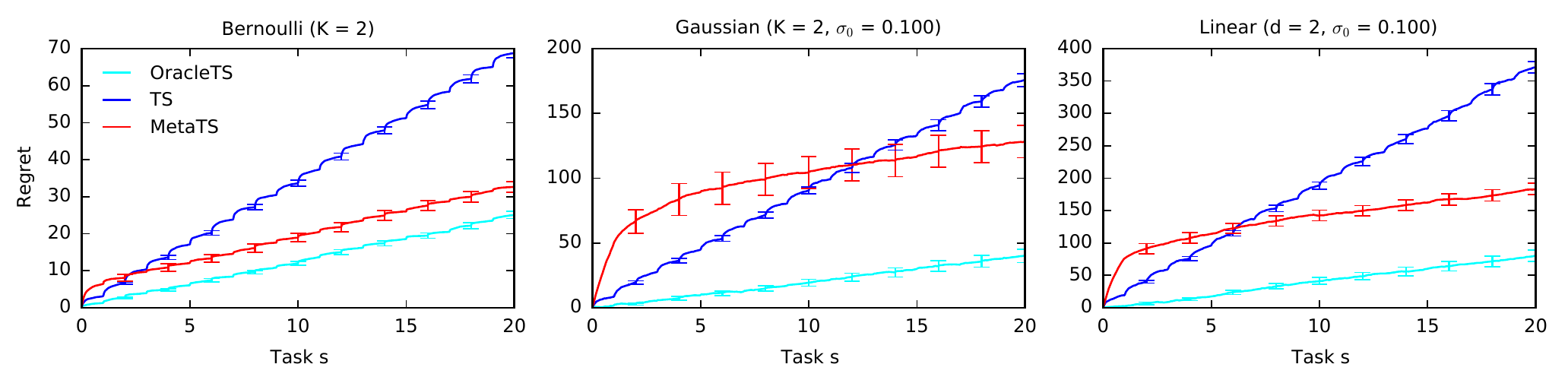}
  \vspace{-0.1in}
  \caption{Comparison of \metats to two variants of Thompson sampling, where the instance prior $P_*$ is known (\oraclets) and the meta-prior $Q$ is marginalized out (\ts).}
  \label{fig:synthetic}
\end{figure*}

The first problem is a Bernoulli bandit with $K = 2$ arms and a categorical meta-prior (\cref{sec:categorical meta-prior}). We have $L = 2$ instance priors, which are defined as
\begin{align*}
  P^{(1)}(\theta)
  & = \mathrm{Beta}(\theta_1; 6, 2) \, \mathrm{Beta}(\theta_2; 2, 6)\,, \\
  P^{(2)}(\theta)
  & = \mathrm{Beta}(\theta_1; 2, 6) \, \mathrm{Beta}(\theta_2; 6, 2)\,.
\end{align*}
In instance prior $P^{(1)}$, arm $1$ is more likely to be optimal than arm $2$, while arm $1$ is more likely to be optimal in prior $P^{(2)}$. The meta-prior is a categorical distribution $\mathrm{Cat}(w)$ where $w = (0.5, 0.5)$. This problem is designed such that if the agent knew $P_*$, it would know the optimal arm with high probability, and could significantly reduce exploration in future tasks.

The second problem is a Gaussian bandit with $K = 2$ arms and a Gaussian meta-prior (\cref{sec:gaussian meta-prior}). The meta-prior width is $\sigma_q = 0.5$, the instance prior width is $\sigma_0 = 0.1$, and the reward noise is $\sigma = 1$. In this problem, $\sigma_q \gg \sigma_0$ and we expect major gains from meta-learning. In particular, based on our discussion in \cref{sec:discussion},
\begin{align*}
  5.85
  & \approx \sqrt{n + \sigma^2 \sigma_0^{-2} K} - \sqrt{\sigma^2 \sigma_0^{-2} K} \\
  & < \sqrt{n + \sigma^2 (\sigma_0^2 + \sigma_q^2)^{-1} K} -
  \sqrt{\sigma^2 (\sigma_0^2 + \sigma_q^2)^{-1} K} \\
  & \approx 11.63\,.
\end{align*}
The third problem is a linear bandit in $d = 2$ dimensions with $K = 5 d$ arms. We sample arm features uniformly at random from $[-0.5, 0.5]^d$. The meta-prior, prior, and noise are set as in the Gaussian experiment. The main difference is that $\theta_{s, *}$ is a parameter vector of a linear model, where the mean reward of arm $x$ is $x\T \theta_{s, *}$. The posterior updates are computed as described in \cref{sec:bayesian multi-task regression}. Even in this more complex setting, they have a closed form.

We compare \metats to two baselines. The first, \oraclets, is idealized TS with the true prior $P_*$. This baseline shows the lowest attainable regret. The second baseline is agnostic TS, which does not use the structure of our problem. In the Gaussian and linear bandit experiments, we implement it as TS with prior $\cN(\theta; \mathbf{0}, (\sigma_q^2 + \sigma_0^2) I_K)$, as this is a marginal distribution of $\theta_{s, *}$. In the Bernoulli bandit experiment, we use an uninformative prior $\prod_{i = 1}^K \mathrm{Beta}(\theta_i; 1, 1)$, since the marginal distribution does not have a closed form. We call this baseline \ts.

Our results are reported in \cref{fig:synthetic}. We plot the cumulative regret as a function of the number of experienced tasks $s$, as it accumulates round-by-round within tasks. The regret of algorithms that do not learn $\mu_*$, such as \oraclets and \ts, is linear in $s$, since they solve $s$ similar tasks using the same policy (\cref{sec:setting}). A lower slope of the regret indicates a better policy. Since \oraclets is optimal in our problems, no algorithm can have sublinear regret in $s$.

In all plots in \cref{fig:synthetic}, we observe significant gains due to meta-learning $P_*$. \metats outperforms \ts, which does not adapt to $P_*$ and performs comparably to \oraclets. This can be seen from the slope of the regret. Specifically, the slope of the \metats regret approaches that of \oraclets after just a few tasks. The slopes of \ts and \oraclets do not change, as these methods do not adapt between tasks.

In \cref{sec:additional experiments}, we report additional experimental results. We observe that the benefits of meta-learning are preserved as the number of arms $K$ or dimensions $d$ increases. However, as is expected, they diminish when the prior width $\sigma_0$ approaches the meta-prior width $\sigma_q$. In this case, there is little benefit from adapting to $P_*$ and all methods perform similarly. We also experiment with misspecified \metats and show that the impact of the misspecification is relatively minor. This attests to the robustness of \metats.

\section{Related Work}
\label{sec:related work}

The closest related work is that of \citet{bastani19meta} who propose TS that learns an instance prior from a sequence of pricing experiments. Their approach is tailored to pricing and learns through forced exploration using a conservative variant of TS, resulting in a meta-learning algorithm that is more conservative and less general than our work. \citet{bastani19meta} also do not derive improved regret bounds due to meta-learning.

\metats is an instance of meta-learning \citep{thrun96explanationbased,thrun98lifelong,baxter98theoretical,baxter00model}, where the agent learns to act under an unknown prior $P_*$ from interactions with bandit instances. Earlier works on a similar topic are \citet{azar13sequential} and \citet{gentile14online}, who proposed UCB algorithms for multi-task learning in the bandit setting. Multi-task learning in contextual bandits, where the arms are similar tasks, was studied by \citet{deshmukh17multitask}. \citet{cella20metalearning} proposed a \linucb algorithm that meta-learns mean parameter vectors in linear models. \citet{yang20provable} studied a setting where the learning agent interacts with multiple bandit instances in parallel and tries to learn their shared subspace. A general template for sequential meta-learning is outlined in \citet{ortega19metalearning}. Our work departs from most of the above approaches in two aspects. First, we have a posterior sampling algorithm that naturally represents the uncertainty in the unknown prior $P_*$. Second, we have a Bayes regret analysis. The shortcoming of the Bayes regret is that it is a weaker optimality criterion than the frequentist regret. To the best of our knowledge, this is the first work to propose meta-learning for Thompson sampling that is natural and has provable guarantees on improvement.

It is well known that the regret of bandit algorithms can be reduced by tuning \citep{vermorel05multiarmed,maes12metalearning,kuleshov14algorithms,hsu19empirical}. All of these works are empirical and focus on the offline setting, where the bandit algorithms are optimized against a known instance distribution. Several recent approaches formulated learning of bandit policies as policy-gradient optimization \citep{duan16rl2,boutilier20differentiable,kveton20differentiable,yang20differentiable,min20policy}. Notably, both \citet{kveton20differentiable} and \citet{min20policy} proposed policy-gradient optimization of TS. These works are in the offline setting and have no global optimality guarantees, except for some special cases \citep{boutilier20differentiable,kveton20differentiable}.

\section{Conclusions}
\label{sec:conclusions}

Thompson sampling \citep{thompson33likelihood}, a very popular and practical bandit algorithm \cite{chapelle11empirical,agrawal12analysis,russo18tutorial}, is parameterized by a prior, which is specified by the algorithm designer. We study a more general setting where the designer can specify an uncertain prior, and the actual prior is learned from sequential interactions with bandit instances. We propose \metats, a computationally-efficient algorithm for this problem. Our analysis of \metats shows the benefit of meta-learning and builds on a novel prior-dependent upper bound on the Bayes regret of Thompson sampling. \metats shows considerable promise in our synthetic experiments.

Our work is a step in the exciting direction of meta-learning state-of-the-art exploration algorithms with guarantees. It has several limitations that should be addressed in future work. First, our regret analysis relies only on a single pull of an arm per task. While this simplifies the analysis and is sufficient to show improvements due to meta-learning, it is inherently conservative. Second, our analysis is limited to Gaussian bandits and relies heavily on the properties of Gaussian posteriors. While we believe that a generalization is possible (\cref{sec:beyond gaussian bandits}), it is likely to be more algebraically demanding. Finally, we hope to analyze our method in contextual bandits. As we show in \cref{sec:bayesian multi-task regression}, meta-posterior updates in linear bandits with Gaussian noise have a closed form. We believe that our work lays foundations for a potential analysis of this approach, which could yield powerful contextual bandit algorithms that adapt to an unknown problem class.

\bibliographystyle{icml2021}
\bibliography{References}

\clearpage
\onecolumn
\appendix

\section{Regret Bound Lemmas}
\label{sec:regret bound lemmas}

\priordependentregretbound*
\begin{proof}
Let $\hat{\theta}_t \in \realset^K$ be the MAP estimate of $\theta_*$ in round $t$, $\theta_t \in \realset^K$ be the posterior sample in round $t$, and $H_t$ denote the history in round $t$. Note that in posterior sampling, $\condprob{\theta_t = \theta}{H_t} = \condprob{\theta_* = \theta}{H_t}$ for all $\theta$. Let $A_*$ be the optimal arm under $\theta_*$ and $A_t$ be the optimal arm under $\theta_t$.

We rely on several properties of Gaussian posterior sampling with a diagonal prior covariance matrix. More specifically, the posterior distribution in round $t$ is $\cN(\hat{\theta}_t, \Sigma_t)$, where $\Sigma_t = \diag{(\sigma_{t, i}^2)_{i = 1}^K}$ is a diagonal covariance matrix with non-zero entries
\begin{align*}
  \sigma_{t, i}^2
  = \frac{1}{\sigma_0^{-2} + N_t(i) \sigma^{-2}}
  = \frac{\sigma^2}{\sigma^2 \sigma_0^{-2} + N_t(i)}\,,
\end{align*}
and $N_t(i)$ denotes the number of pulls of arm $i$ up to round $t$. Accordingly, a high-probability confidence interval of arm $i$ in round $t$ is $C_t(i) = \sqrt{2 \sigma_{t, i}^2 \log(1 / \delta)}$, where $\delta > 0$ is the confidence level. Let
\begin{align*}
  E_t
  = \set{\forall i \in [K]: \abs{\theta_*(i) -  \hat{\theta}_t(i)} \leq C_t(i)}
\end{align*}
be the event that all confidence intervals in round $t$ hold.

Now we bound the regret in round $t$. Fix round $t$. The regret can be decomposed as
\begin{align*}
  \E{\theta_*(A_*) - \theta_*(A_t)}
  & = \E{\condE{\theta_*(A_*) - \theta_*(A_t)}{H_t}} \\
  & = \E{\condE{\theta_*(A_*) - \hat{\theta}_t(A_*) - C_t(A_*)}{H_t}} +
  \E{\condE{\hat{\theta}_t(A_t) + C_t(A_t) - \theta_*(A_t)}{H_t}}\,.
\end{align*}
The first equality is an application of the tower rule. The second equality holds because $A_t \mid H_t$ and $A_* \mid H_t$ have the same distributions, and $\hat{\theta}_t$ and $C_t$ are deterministic given history $H_t$.

We start with the first term in the decomposition. Fix history $H_t$. Then we introduce event $E_t$ and get
\begin{align*}
  \condE{\theta_*(A_*) - \hat{\theta}_t(A_*) - C_t(A_*)}{H_t}
  & = \condE{\theta_*(A_*) - \hat{\theta}_t(A_*)}{H_t} - \condE{C_t(A_*)}{H_t} \\
  & \leq \condE{(\theta_*(A_*) - \hat{\theta}_t(A_*)) \I{\bar{E}_t}}{H_t}\,,
\end{align*}
where the inequality follows from the observation that
\begin{align*}
  \condE{(\theta_*(A_*) - \hat{\theta}_t(A_*)) \I{E_t}}{H_t}
  \leq \condE{C_t(A_*)}{H_t}\,.
\end{align*}
Since $\theta_* - \hat{\theta}_t \mid H_t \sim \cN(\mathbf{0}, \Sigma_t)$, we further have
\begin{align}
  \condE{(\theta_*(A_*) - \hat{\theta}_t(A_*)) \I{\bar{E}_t}}{H_t}
  & \leq \sum_{i = 1}^K \frac{1}{\sqrt{2 \pi \sigma_{t, i}^2}}
  \int_{x = C_t(i)}^\infty x
  \exp\left[- \frac{x^2}{2 \sigma_{t, i}^2}\right] \dif x
  \nonumber \\
  & = \sum_{i = 1}^K - \sqrt{\frac{\sigma_{t, i}^2}{2 \pi}}
  \int_{x = C_t(i)}^\infty \frac{\partial}{\partial x}
  \left(\exp\left[- \frac{x^2}{2 \sigma_{t, i}^2}\right]\right) \dif x
  \nonumber \\
  & = \sum_{i = 1}^K \sqrt{\frac{\sigma_{t, i}^2}{2 \pi}} \delta
  \leq \sqrt{\frac{\sigma_0^2}{2 \pi}} K \delta\,.
  \label{eq:bayes regret scale}
\end{align}
For the second the term in the regret decomposition, we have
\begin{align*}
  \condE{\hat{\theta}_t(A_t) + C_t(A_t) - \theta_*(A_t)}{H_t}
  \leq 2 \condE{C_t(A_t)}{H_t} +
  \condE{(\hat{\theta}_t(A_t) - \theta_*(A_t)) \I{\bar{E}_t}}{H_t}\,,
\end{align*}
where the inequality follows from the observation that
\begin{align*}
  \condE{(\hat{\theta}_t(A_t) - \theta_*(A_t)) \I{E_t}}{H_t}
  \leq \condE{C_t(A_t)}{H_t}\,.
\end{align*}
The other term is bounded as in \eqref{eq:bayes regret scale}. Now we chain all inequalities for the regret in round $t$ and get
\begin{align*}
  \E{\theta_*(A_*) - \theta_*(A_t)}
  \leq 2 \E{C_t(A_t)} + \sqrt{\frac{2 \sigma_0^2}{\pi}} K \delta\,.
\end{align*}
Therefore, the $n$-round Bayes regret is bounded as
\begin{align*}
  \E{\sum_{t = 1}^n \theta_*(A_*) - \theta_*(A_t)}
  \leq 2 \E{\sum_{t = 1}^n C_t(A_t)} +
  \sqrt{\frac{2 \sigma_0^2}{\pi}} K n \delta\,.
\end{align*}
The last part is to bound $\E{\sum_{t = 1}^n C_t(A_t)}$ from above. Since the confidence interval $C_t(i)$ decreases with each pull of arm $i$, $\sum_{t = 1}^n C_t(A_t)$ is bounded for any $\theta_*$ by pulling arms in a round robin \citep{russo14learning}, which yields
\begin{align*}
  \E{\sum_{t = 1}^n C_t(A_t)}
  & \leq \sqrt{2 \sigma^2 \log(1 / \delta)} K
  \sum_{s = 0}^{\floors{n / K}} \sqrt{\frac{1}{s + \sigma^2 \sigma_0^{-2}}} \\
  & = \sqrt{2 \sigma^2 \log(1 / \delta)} K
  \left(\sum_{s = 1}^{\floors{n / K}} \sqrt{\frac{1}{s + \sigma^2 \sigma_0^{-2}}} +
  \frac{\sigma_0}{\sigma}\right) \\
  & \leq 2 \sqrt{2 \sigma^2 \log(1 / \delta)} K
  \left(\sqrt{\frac{n}{K} + \sigma^2 \sigma_0^{-2}} -
  \sqrt{\sigma^2 \sigma_0^{-2}}\right) +
  \sqrt{2 \sigma_0^2 \log(1 / \delta)} K \\
  & = 2 \sqrt{2 \sigma^2 K \log(1 / \delta)}
  \left(\sqrt{n + \sigma^2 \sigma_0^{-2} K} -
  \sqrt{\sigma^2 \sigma_0^{-2} K}\right) +
  \sqrt{2 \sigma_0^2 \log(1 / \delta)} K\,.
\end{align*}
The second inequality follows from \cref{lem:reciprocal root sum}. Now we chain all inequalities and this completes the proof.
\end{proof}

\misspecifiedprior*
\begin{proof}
First, we bound the regret when $\theta_*$ is not close to $\mu_*$. Let
\begin{align*}
  E
  = \set{\forall i \in [K]: \abs{\theta_*(i) -  \mu_*(i)} \leq c}
\end{align*}
be the event that $\theta_*$ is close to $\mu_*$, where $c = \sqrt{2 \sigma_0^2 \log(1 / \delta)}$ is the corresponding confidence interval. Then
\begin{align*}
  \E{\I{\bar{E}} \sum_{t = 1}^n \theta_*(\hat{A}_t) - \theta_*(\tilde{A}_t)}
  \leq 2 n \E{\I{\bar{E}} \max_{i \in [K]} \abs{\theta_*(i)}}\,.
\end{align*}
The expectation can be further bounded as
\begin{align*}
  \E{\I{\bar{E}} \max_{i \in [K]} \abs{\theta_*(i)}}
  & \leq \sum_{i = 1}^K
  \frac{1}{\sqrt{2 \pi \sigma_0^2}} \int_{x = c}^\infty
  x \exp\left[- \frac{x^2}{2 \sigma_0^2}\right] \dif x +
  \abs{\mu_*(i)} \exp\left[- \frac{c^2}{2 \sigma_0^2}\right] \\
  & = \sum_{i = 1}^K
  - \sqrt{\frac{\sigma_0^2}{2 \pi}} \int_{x = c}^\infty \frac{\partial}{\partial x}
  \left(\exp\left[- \frac{x^2}{2 \sigma_0^2}\right]\right) \dif x +
  \abs{\mu_*(i)} \exp\left[- \frac{c^2}{2 \sigma_0^2}\right] \\
  & \leq \left(\sqrt{\frac{\sigma_0^2}{2 \pi}} + \maxnorm{\mu_*}\right) K \delta\,.
\end{align*}
In the first inequality, we use that $\theta_*(i) -  \mu_*(i)$ is $\sigma_0^2$-sub-Gaussian. Now we combine the above two inequalities and get
\begin{align*}
  \E{\sum_{t = 1}^n \theta_*(\hat{A}_t) - \theta_*(\tilde{A}_t)}
  & \leq \E{\I{E} \sum_{t = 1}^n \theta_*(\hat{A}_t) - \theta_*(\tilde{A}_t)} +
  4 \left(\sqrt{\frac{\sigma_0^2}{2 \pi}} + \maxnorm{\mu_*}\right) K n \delta\,.
\end{align*}
The main challenge in bounding the first term above is that the posterior distributions of $\hat{A}_t$ and $\tilde{A}_t$ may deviate significantly if their histories do.

We get a bound that depends on the difference of prior means $\eps$ based on the following observation. In round $1$, both TS algorithms behave identically, on average over the posterior samples, in $\sum_{i = 1}^K \min \set{\prob{\hat{A}_1 = i}, \, \prob{\tilde{A}_1 = i}}$ fraction of runs. In this case, the expected regret in round $1$ is zero and the algorithms have the same history distributions in round $2$, on average over the posterior samples. Otherwise, in $\sum_{i = 1}^K \abs{\prob{\hat{A}_1 = i} - \prob{\tilde{A}_1 = i}}$ fraction of runs, the algorithms behave differently and we bound the difference of their future rewards trivially by $2 (\maxnorm{\mu_*} + c) n$. Now we apply this bound from round $2$ to $n$, conditioned on both algorithms having the same history distributions, and get
\begin{align*}
  \E{\I{E} \sum_{t = 1}^n \theta_*(\hat{A}_t) - \theta_*(\tilde{A}_t)}
  & \leq 2 (\maxnorm{\mu_*} + c) n^2
  \max_{t \in [n], \, h \in \cH_t} \sum_{i = 1}^K
  \abs{\condprob{\hat{A}_t = i}{\hat{H}_t = h} -
  \condprob{\tilde{A}_t = i}{\tilde{H}_t = h}}\,,
\end{align*}
where $\hat{H}_t$ is the history for $\hat{A}_t$, $\tilde{H}_t$ is the history for $\tilde{A}_t$, and $\cH_t$ is the set of all possible histories in round $t$. Finally, we bound the last term above using $\eps$.

Fix round $t$ and history $h \in \cH_t$. Let $p(\theta) = \condprob{\hat{\theta}_t = \theta}{\hat{H}_t = h}$ and $q(\theta) = \condprob{\tilde{\theta}_t = \theta}{\tilde{H}_t = h}$. Then, since the pulled arms are deterministic functions of their posterior samples, we have
\begin{align*}
  \sum_{i = 1}^K
  \abs{\condprob{\hat{A}_t = i}{\hat{H}_t = h} -
  \condprob{\tilde{A}_t = i}{\tilde{H}_t = h}}
  \leq \int_\theta \abs{p(\theta) - q(\theta)} \dif \theta\,.
\end{align*}
Moreover, since the posterior distributions $p(\theta) = \prod_{i = 1}^K p(\theta_i)$ and $q(\theta) = \prod_{i = 1}^K q(\theta_i)$ are factored, we get
\begin{align*}
  \abs{p(\theta) - q(\theta)}
  & = \abs{\prod_{i = 1}^K p(\theta_i) - \prod_{i = 1}^K q(\theta_i)}
  = \abs{\prod_{i = 1}^K p(\theta_i) -
  q(\theta_1) \prod_{i = 2}^K p(\theta_i) +
  q(\theta_1) \prod_{i = 2}^K p(\theta_i) -
  \prod_{i = 1}^K q(\theta_i)} \\
  & \leq \abs{p(\theta_1) - q(\theta_1)} \prod_{i = 2}^K p(\theta_i) +
  q(\theta_1) \abs{\prod_{i = 2}^K p(\theta_i) - \prod_{i = 2}^K q(\theta_i)} \\
  & \leq \sum_{i = 1}^K \left(\prod_{j = 1}^{i - 1} q(\theta_j)\right)
  \abs{p(\theta_i) - q(\theta_i)}
  \left(\prod_{j = i + 1}^K p(\theta_j)\right)\,.
\end{align*}
The last inequality follows from the recursive application of the decomposition. Because of the above factored structure, the integral factors as
\begin{align*}
  \int_\theta \abs{p(\theta) - q(\theta)} \dif \theta
  \leq \sum_{i = 1}^K \int_{\theta_i} \abs{p(\theta_i) - q(\theta_i)} \dif \theta_i\,.
\end{align*}
Let $\hat{\mu}$ and $\tilde{\mu}$ be the means of $p$ and $q$, respectively. Let $s$ be the number of pulls of arm $i$ and $\hat{\sigma}^2 = (\sigma_0^{-2} + s \sigma^{-2})^{-1}$ be the posterior variance, which is the same for both $p$ and $q$. Then, under the assumption that the prior means differ by at most $\eps$ in each entry, each above integral is bounded as
\begin{align*}
  \int_{\theta_i} \abs{p(\theta_i) - q(\theta_i)} \dif \theta_i
  \leq \frac{2}{\sqrt{2 \pi \hat{\sigma}^2}} \abs{\hat{\mu} - \tilde{\mu}}
  \leq \frac{2}{\sqrt{2 \pi \hat{\sigma}^2}} \frac{\hat{\sigma}^2 \eps}{\sigma_0^2}
  \leq \sqrt{\frac{2}{\pi \sigma_0^2}} \eps\,.
\end{align*}
The first inequality holds for any two shifted non-negative unimodal functions, such as $p$ and $q$, with maximum $1 / \sqrt{2 \pi \hat{\sigma}^2}$. The second inequality is from the fact that $\hat{\mu}$ and $\tilde{\mu}$ are estimated from the same $s$ observations, and that the difference of their prior means is at most $\eps$. The last inequality holds because $\hat{\sigma} \leq \sigma_0$.

Finally, we chain all inequalities and get our claim.
\end{proof}

\metaposteriorconcentration*
\begin{proof}
The key idea in the proof is that $\mu_* \mid H_{1 : s - 1} \sim \cN(\hat{\mu}_s, \hat{\Sigma}_s)$ and $\tilde{\mu}_s \mid H_{1 : s - 1} \sim \cN(\hat{\mu}_s, \hat{\Sigma}_s)$, where $\hat{\Sigma}_s$ is a diagonal covariance matrix. We focus on analyzing $\mu_*$ first.

To simplify notation, let $\mu_* = (\mu_{*, i})_{i = 1}^K$, $\hat{\mu}_s = (\hat{\mu}_{s, i})_{i = 1}^K$, and $\hat{\Sigma}_s = \diag{(\hat{\sigma}_{s, i}^2)_{i = 1}^K}$. Fix task $s$ and history $H_{1 : s - 1}$. Then, from the definition of $\mu_* \mid H_{1 : s - 1}$, we have for any $\eps > 0$ that
\begin{align*}
  \condprob{\maxnorm{\mu_* - \hat{\mu}_s} > \eps}{H_{1 : s - 1}}
  & \leq \sum_{i = 1}^K
  \condprob{\abs{\mu_{*, i} - \hat{\mu}_{s, i}} > \eps}{H_{1 : s - 1}}
  \leq 2 \sum_{i = 1}^K
  \exp\left[- \frac{\eps^2}{2 \hat{\sigma}_{s, i}^2}\right] \\
  & \leq 2 \sum_{i = 1}^K \exp\left[- \frac{\eps^2}{2}
  \left(\sigma_q^{-2} + \frac{s - 1}{\sigma_0^2 + \sigma^2}\right)\right]
  = 2 K \exp\left[- \frac{\eps^2}{2}
  \left(\sigma_q^{-2} + \frac{s - 1}{\sigma_0^2 + \sigma^2}\right)\right]\,.
\end{align*}
The third inequality is from a trivial upper bound on $\hat{\sigma}_{s, i}^2$, which holds because each arm is pulled at least once per task. Now we choose
\begin{align*}
  \eps_s
  = \sqrt{2 \left(\sigma_q^{-2} + \frac{s - 1}{\sigma_0^2 + \sigma^2}\right)^{-1}
  \log(4 K / \delta)}
  = \sqrt{2 \frac{\sigma_0^2 + \sigma^2}{(\sigma_0^2 + \sigma^2) \sigma_q^{-2} + s - 1}
  \log(4 K / \delta)}
\end{align*}
and get that $\condprob{\maxnorm{\mu_* - \hat{\mu}_s} > \eps_s}{H_{1 : s - 1}} \leq \delta / 2$ for any task $s$ and history $H_{1 : s - 1}$. It follows that
\begin{align*}
  \prob{\bigcup_{s = 1}^m \set{\maxnorm{\mu_* - \hat{\mu}_s} > \eps_s}}
  \leq \sum_{s = 1}^m \prob{\maxnorm{\mu_* - \hat{\mu}_s} > \eps_s}
  = \sum_{s = 1}^m \E{\condprob{\maxnorm{\mu_* - \hat{\mu}_s} > \eps_s}{H_{1 : s - 1}}}
  \leq \frac{m \delta}{2}\,.
\end{align*}
Since $\tilde{\mu} \mid H_{1 : s - 1}$ is distributed identically to $\mu_* \mid H_{1 : s - 1}$, we have from the same line of reasoning that
\begin{align*}
  \prob{\bigcup_{s = 1}^m \set{\maxnorm{\tilde{\mu}_s - \hat{\mu}_s} > \eps_s}}
  \leq \frac{m \delta}{2}\,.
\end{align*}
Finally, we apply the triangle inequality and union bound,
\begin{align*}
  \prob{\bigcup_{s = 1}^m \set{\maxnorm{\tilde{\mu}_s - \mu_*} > \eps_s}}
  \leq \prob{\bigcup_{s = 1}^m \set{\maxnorm{\tilde{\mu}_s - \hat{\mu}_s} > \frac{\eps_s}{2}}} +
  \prob{\bigcup_{s = 1}^m \set{\maxnorm{\mu_* - \hat{\mu}_s} > \frac{\eps_s}{2}}}\,,
\end{align*}
and then double the value of $\eps_s$.
\end{proof}

\section{Technical Lemmas}
\label{sec:technical lemmas}

\begin{lemma}
\label{lem:gaussian posterior update} Let $\mu_0 \sim \cN(\hat{\mu}, \hat{\sigma}^2)$, $\theta \mid \mu_0 \sim \cN(\mu_0, \sigma_0^2)$, and $Y_i \mid \theta \sim \cN(\theta, \sigma^2)$ for all $i \in [n]$. Then
\begin{align*}
  \mu_0 \mid Y_1, \dots, Y_n
  \sim \cN\left(\lambda^{-1} \left(\frac{\hat{\mu}}{\hat{\sigma}^2} +
  \frac{\sum_{i = 1}^n Y_i}{n \sigma_0^2 + \sigma^2}\right), \,
  \lambda^{-1}\right)\,, \quad
  \lambda
  = \frac{1}{\hat{\sigma}^2} + \left(\sigma_0^2 + \frac{\sigma^2}{n}\right)^{-1}\,.
\end{align*}
\end{lemma}
\begin{proof}
The derivation is standard \citep{gelman07data} and we only include it for completeness. To simplify notation, let
\begin{align*}
  v
  = \sigma^{-2}\,, \quad
  v_0
  = \sigma_0^{-2}\,, \quad
  \hat{v}
  = \hat{\sigma}^{-2}\,, \quad
  c_1
  = v_0 + n v\,, \quad
  c_2
  = \hat{v} + v_0 - c_1^{-1} v_0^2\,.
\end{align*}
The posterior distribution of $\mu_0$ is
\begin{align*}
  & \int_\theta \left(\prod_{i = 1}^n \cN(Y_i; \theta, \sigma^2)\right)
  \cN(\theta; \mu_0, \sigma_0^2) \dif \theta \,
  \cN(\mu_0; \hat{\mu}, \hat{\sigma}^2) \\
  & \quad \propto \int_\theta \exp\left[
  - \frac{1}{2} v \sum_{i = 1}^n (Y_i - \theta)^2 -
  \frac{1}{2} v_0 (\theta - \mu_0)^2\right] \dif \theta
  \exp\left[- \frac{1}{2} \hat{v} (\mu_0 - \hat{\mu})^2\right]\,.
\end{align*}
Let $f(\mu_0)$ denote the integral. We solve it as
\begin{align*}
  f(\mu_0)
  & = \int_\theta \exp\left[- \frac{1}{2}
  \left(v \sum_{i = 1}^n (Y_i^2 - 2 Y_i \theta + \theta^2) +
  v_0 (\theta^2 - 2 \theta \mu_0 + \mu_0^2)\right)\right] \dif \theta \\
  & \propto \int_\theta \exp\left[- \frac{1}{2}
  \left(c_1 \left(\theta^2 -
  2 c_1^{-1} \theta \left(v \sum_{i = 1}^n Y_i + v_0 \mu_0\right)\right) +
  v_0 \mu_0^2\right)\right] \dif \theta \\
  & = \int_\theta \exp\left[- \frac{1}{2}
  \left(c_1 \left(\theta -
  c_1^{-1} \left(v \sum_{i = 1}^n Y_i + v_0 \mu_0\right)\right)^2 -
  c_1^{-1} \left(v \sum_{i = 1}^n Y_i + v_0 \mu_0\right)^2 +
  v_0 \mu_0^2\right)\right] \dif \theta \\
  & = \exp\left[- \frac{1}{2}
  \left(- c_1^{-1} \left(v \sum_{i = 1}^n Y_i + v_0 \mu_0\right)^2 +
  v_0 \mu_0^2\right)\right] \\
  & \propto \exp\left[- \frac{1}{2}
  \left(- c_1^{-1} \left(2 v_0 \mu_0 v \sum_{i = 1}^n Y_i + v_0^2 \mu_0^2\right) +
  v_0 \mu_0^2\right)\right]\,.
\end{align*}
Now we chain all equalities and have
\begin{align*}
  f(\mu_0) \exp\left[- \frac{1}{2} \hat{v} (\mu_0 - \hat{\mu})^2\right]
  & \propto \exp\left[- \frac{1}{2}
  \left(- c_1^{-1} \left(2 v_0 \mu_0 v \sum_{i = 1}^n Y_i + v_0^2 \mu_0^2\right) +
  v_0 \mu_0^2 + \hat{v} (\mu_0^2 - 2 \mu_0 \hat{\mu} + \hat{\mu}^2)
  \right)\right] \\
  & \propto \exp\left[- \frac{1}{2}
  c_2 \left(\mu_0^2 -
  2 c_2^{-1} \mu_0 \left(c_1^{-1} v_0 v \sum_{i = 1}^n Y_i +
  \hat{v} \hat{\mu}\right)\right)\right] \\
  & \propto \exp\left[- \frac{1}{2}
  c_2 \left(\mu_0 -
  c_2^{-1} \left(c_1^{-1} v_0 v \sum_{i = 1}^n Y_i +
  \hat{v} \hat{\mu}\right)\right)^2\right]\,.
\end{align*}
Finally, note that
\begin{align*}
  c_1
  & = \frac{1}{\sigma_0^2} + \frac{n}{\sigma^2}
  = \frac{n \sigma_0^2 + \sigma^2}{\sigma_0^2 \sigma^2}\,, \\
  c_1^{-1} v_0 v
  & = \frac{\sigma_0^2 \sigma^2}{n \sigma_0^2 + \sigma^2} \frac{1}{\sigma_0^2 \sigma^2}
  = \frac{1}{n \sigma_0^2 + \sigma^2}\,, \\
  c_2
  & = \frac{1}{\hat{\sigma}^2} +
  \frac{1}{\sigma_0^2} \left(1 - \frac{c_1^{-1}}{\sigma_0^2}\right)
  = \frac{1}{\hat{\sigma}^2} +
  \frac{1}{\sigma_0^2} \left(1 - \frac{\sigma^2}{n \sigma_0^2 + \sigma^2}\right)
  = \frac{1}{\hat{\sigma}^2} + \left(\sigma_0^2 + \frac{\sigma^2}{n}\right)^{-1}
  = \lambda\,.
\end{align*}
This concludes the proof.
\end{proof}

\begin{lemma}
\label{lem:reciprocal root sum} For any integer $n$ and $a \geq 0$,
\begin{align*}
  \sum_{i = 1}^n \frac{1}{\sqrt{i + a}}
  \leq 2 (\sqrt{n + a} - \sqrt{a})
  \leq 2 \sqrt{n}\,.
\end{align*}
\end{lemma}
\begin{proof}
Since $1 / \sqrt{i + a}$ decreases in $i$, the sum can be bounded using integration as
\begin{align*}
  \sum_{i = 1}^n \frac{1}{\sqrt{i + a}}
  \leq \int_{x = a}^{n + a} \frac{1}{\sqrt{x}} \dif x
  = 2 (\sqrt{n + a} - \sqrt{a})\,.
\end{align*}
The inequality $\sqrt{n + a} - \sqrt{a} \leq \sqrt{n}$ holds because the square root has diminishing returns.
\end{proof}

\begin{lemma}
\label{lem:reciprocal sum} For any integer $n$ and $a \geq 0$,
\begin{align*}
  \sum_{i = 1}^n \frac{1}{i + a}
  \leq \log(1 + n / a)\,.
\end{align*}
\end{lemma}
\begin{proof}
Since $1 / (i + a)$ decreases in $i$, the sum can be bounded using integration as
\begin{align*}
  \sum_{i = 1}^n \frac{1}{i + a}
  \leq \int_{x = a}^{n + a} \frac{1}{x} \dif x
  = \log(n + a) - \log a
  = \log(1 + n / a)\,.
\end{align*}
\end{proof}

\section{Measure-Theoretic View}
\label{sec:measure-theoretic-view}

\measuretheorylemma*
\begin{proof}
Note that by Ionescu-Tulcea theorem
we have a well-defined measure
$\mathbb{P}_{z,\theta_s,\pi}$
and the Radon-Nikodym derivative of 
$\mathbb{P}_{z,\theta_s,\pi}$ 
with respect to the 
the $\sigma$-finite measure
$(Q\times\lambda_\rho\times\kappa)^{n}$
where $\kappa$ denotes the counting measure is:
\begin{align*}
p_{\theta_s,\pi}(a_{s,1},x_{s,1},\dots,
a_{s,n},x_{s,n}) = 
\prod_{t=1}^n
\pi_{s,t}(a_{s,t}|
a_{s,1},x_{s,1},\dots,a_{s,t-1},x_{s,t-1})
p_{\theta_s,a_{s,t}}(x_{s,t}).
\end{align*}
Further, marginalizing over the 
the environments $\theta_1,\dots,\theta_m$ to
get the measure $\mathbb{P}_{z,\pi}$ and 
taking its derivative with respect to 
the $\sigma$-finite measure
$(\lambda_\rho\times\kappa)^{m\times n}$ 
gives the following:
\begin{align*}
p_{\theta,\pi}(a_{1,1},x_{1,1},
\dots,a_{m,t},x_{m,t}) =
\prod_{s=1}^m 
\prod_{t=1}^n\pi_{s,t}(a_{s,t}|a_{1,1},x_{1,1},
\dots,a_{s,t-1},x_{s,t-1})
\int_{\theta\in\Theta}
p_{\theta,a_{s,t}}(x_{s,t}) \dif P_z(\theta).
\end{align*}

For a set $S_1$ which 
is an element of the $\sigma$-algebra $\sigma(\Theta)$
over the space of the environments $\Theta$,
$S_1\in\sigma(\Theta)$,
and any hidden variable $z$ from the 
space of outcomes of the meta-prior, $z\in\cZ$, the 
task-posterior after $t$ rounds in task 
$s$ could be written as:
\begin{align*}
P_{z}(
S_1|a_{s,1},x_{s,1},\dots,a_{s,t},x_{s,t})
&= \frac{\int_{S_1} p_{\theta_s,\pi}(
a_{s,1},x_{s,1},\dots,a_{s,t},x_{s,t})
\dif P_{z}(\theta_s)}{\int_\Theta
p_{\theta_s,\pi}(
a_{s,1},x_{s,1},\dots, a_{s,t},x_{s,t})
\dif P_{z}(\theta_s)} \\
&= \frac{\int_{S_1} 
\prod_{j=1}^{t-1} p_{\theta_s,a_{s,j}}(x_{s,j})
\dif P_{z}(\theta_s)
}{\int_\Theta
\prod_{j=1}^{t-1} p_{\theta_s,a_{s,j}}(x_{s,j})
\dif P_{z}(\theta_s)}.
\end{align*}
Further, for any set $S_2$ which
is an element of the $\sigma$-algebra $\sigma(\cZ)$
over the space of hidden variables $\cZ$,
$S_2\in\sigma(\cZ)$,
the meta-posterior after $t$ rounds in task
$s$ could be written as: 
\begin{align*}
Q(S_2|a_{1,1},x_{1,1},\dots,a_{s-1,n},x_{s-1,n})
&= \frac{\int_{S_2} p_{z,\pi}(
a_{1,1},x_{1,1},\dots,a_{s-1,n},x_{s-1,n})
\dif Q(z)}{
\int_{\cZ} p_{z,\pi}(a_{1,1},x_{1,1},\dots
,a_{s-1,n},x_{s-1,n}) \dif Q(z)}\\
&= \frac{
\int_{S_2} 
\left[\prod_{\ell=1}^{s-1} \prod_{j=1}^n
\int_{\Theta}p_{
\theta,a_{\ell,j}}(x_{\ell,j})
\dif P_z(\theta)
\right] \dif Q(z)}{
\int_\cZ \left[\prod_{\ell=1}^{s-1} \prod_{j=1}^n
\int_{\Theta}p_{
\theta,a_{\ell,j}}(x_{\ell,j})
\dif P_z(\theta)
\right]\dif Q(z)}.
\end{align*}
\end{proof}

\section{Bayesian Multi-Task Regression}
\label{sec:bayesian multi-task regression}

Multi-task Bayesian regression was proposed by \citet{lindley72bayes}. In their work, there is only one prior for all the regression tasks. In other words, there is no meta-prior and the posterior is calculated only for $2$ levels. In our work, we extended this formulation by adding a meta-prior. Now we need to compute the posterior across $3$ levels, which we work out in this section. Recall that all variances in our hierarchical data generation process are known. Formally, we have the following generative model
\begin{equation}
\begin{aligned}
  \theta_0 \;
  &\sim \; \mathcal{N}(\mu_0, \Lambda_0^{-1}) \\
  \theta_{t}|\theta_0 \;
  &\sim \; \mathcal{N}(\theta_0, \Sigma) \qquad\;\, \text{ for } t=1,2,\dots,T \\
  y_{t,i}|\theta_t, x_{t,i} \;
  &\sim \; \mathcal{N}(x_{t,i}^T \theta_t, \sigma^2) \quad \text{ for } i=1,2,\dots,N
\end{aligned}
\end{equation}
For convenience, we group all observations in a matrix-vector form as $X_t = (x_{t,1},\dots,x_{t,n})$ and $y_t = (y_{t,1},\dots,y_{t,n})$. Our goal is to compute $p\left(\theta_0 \middle| X_{1:T}, y_{1:T}\right)$. We begin by trying to recursively compute
\begin{equation}
  p(\theta_0|X_{1:t}, y_{1:t})
  \propto p(\theta_0|X_{1:t-1}, y_{1:t-1}) p(y_{t} | \theta_0, X_{t})
  \label{eq:recurse}
\end{equation}

\subsection{Computing $p(y_{t} | \theta_0, X_{t})$}

Note that, given $\theta_t$ and $X_t$, we have
\begin{equation}
    y_t = X_t\theta_t + \epsilon \qquad \text{ where } \epsilon \sim \mathcal{N}(0, \sigma^2 I)
\end{equation}
Similarly, given $\theta_0$, we have
\begin{equation}
    \theta_t = \theta_0 + \nu \qquad \text{ where } \nu \sim \mathcal{N}(0, \Sigma)
\end{equation}
Combining the above two equations
\begin{equation}
    y_t = X_t\theta_0 + X_t \nu + \epsilon
\end{equation}
From the above, it is clear that $y_t|\theta_0, X_{t}$ is a Gaussian. We can find the parameters of this distribution easily as
\begin{equation}
\begin{aligned}
    \mathbb{E}[y_t|\theta_0, X_t] &= X_t\theta_0\\
    \text{Cov}(y_t|\theta_0, X_t) &= \mathbb{E}[(y_t-X_t\theta_0)(y_t-X_t\theta_0)^T] \\
    &= \mathbb{E}[(X_t \nu + \epsilon)(X_t \nu + \epsilon)^T] \\
    &= \mathbb{E}[(X_t \nu\nu^TX_t^T + \epsilon \nu^T X_t^T + X_t \nu \epsilon^T + \epsilon\epsilon^T)] \\
    &= X_t \underbrace{\mathbb{E}[\nu\nu^T]}_{=\Sigma}X_t^T + \underbrace{\mathbb{E}[\epsilon \nu^T]}_{=0 \text{ as } \epsilon \perp \nu} X_t + X_t \underbrace{\mathbb{E}[\nu \epsilon^T]}_{=0 \text{ as } \epsilon \perp \nu} + \underbrace{\mathbb{E}[\epsilon\epsilon^T]}_{=\sigma^2 I} \\
    &= X_t\Sigma X_t^T + \sigma^2I
\end{aligned}
\end{equation}
Thus we get
\begin{equation}
    y_{t} | \theta_0, X_{t} \; \sim \; \mathcal{N}(X_t\theta_0, \  \sigma^2I + X_t\Sigma X_t^T)
    \label{eq:marginal_final}
\end{equation}

\subsection{Computing $p(\theta_0|X_{1:t}, y_{1:t})$ by induction}

\paragraph{Induction hypothesis:} $\forall t: \theta_0|X_{1:t}, y_{1:t} \; \sim \; \mathcal{N}(\mu_t, \Lambda_t^{-1}) $

\paragraph{Base case $t=0$:} This corresponds to prior as there is no data and thus by definition $\theta_0 \; \sim \; \mathcal{N}(\mu_0, \Lambda_0^{-1})$.

\paragraph{Inductive step} We derive the distribution of $\theta_0|X_{1:t}, y_{1:t}$ given that $\theta_0|X_{1:t-1}, y_{1:t-1} \; \sim \; \mathcal{N}(\mu_{t-1}, \Lambda_{t-1}^{-1})$. Then by using the recurrence from \eqref{eq:recurse} and marginal distribution from \eqref{eq:marginal_final}, we obtain
\begin{equation}
\begin{aligned}
    p(\theta_0|X_{1:t}, y_{1:t}) &\propto p(\theta_0|X_{1:t-1}, y_{1:t-1}) p(y_{t} | \theta_0, X_{t}) \\
    &\propto \exp\left\{ -\frac{1}{2} (\theta_0-\mu_{t-1})^T\Lambda_{t-1}(\theta_0-\mu_{t-1})  \right\}  \exp\left\{ -\frac{1}{2} (y_t-X_t\theta_0)^T(\sigma^2I + X_t\Sigma X_t^T)^{-1}(y_t-X_t\theta_0)  \right\} \\
    &\propto \exp\left\{ -\frac{1}{2} (\theta_0-\mu_{t-1})^T\Lambda_{t-1}(\theta_0-\mu_{t-1})   -\frac{1}{2} (y_t-X_t\theta_0)^T(\sigma^2I + X_t\Sigma X_t^T)^{-1}(y_t-X_t\theta_0)  \right\} \\
    &\propto \exp\left\{ -\frac{1}{2}\theta_0^T\Lambda_{t-1}\theta_0 +\theta_0^T\Lambda_{t-1}\mu_{t-1} - \frac{1}{2} \theta_0^TX_t^T(\sigma^2I + X_t\Sigma X_t^T)^{-1}X_t\theta_0 + \theta_0^TX_t^T(\sigma^2I + X_t\Sigma X_t^T)^{-1}y_t \right\} \\
    &\propto \exp\left\{ -\frac{1}{2}\theta_0^T\underbrace{\left(\Lambda_{t-1} + X_t^T(\sigma^2I + X_t\Sigma X_t^T)^{-1}X_t\right)}_{=\Lambda_t}\theta_0 +\theta_0^T\underbrace{\left(\Lambda_{t-1}\mu_{t-1} + X_t^T(\sigma^2I + X_t\Sigma X_t^T)^{-1}y_t \right)}_{=\Lambda_t\mu_t} \right\} \\
    &\propto \exp\left\{ -\frac{1}{2}\theta_0^T\Lambda_t\theta_0 +\theta_0^T\Lambda_t\mu_t \right\} \\
    &\propto \exp\left\{ -\frac{1}{2}\theta_0^T\Lambda_t\theta_0 +\theta_0^T\Lambda_t\mu_t -\frac{1}{2}\mu_t^T\Lambda_t\mu_t  \right\} \\
    &\propto \exp\left\{ -\frac{1}{2}(\theta_0-\mu_t)^T\Lambda_t(\theta_0-\mu_t) \right\} \\
\end{aligned}
\end{equation}
This completes our proof by induction, as now we have shown that $\theta_0|X_{1:t}, y_{1:t} \; \sim \; \mathcal{N}(\mu_t, \Lambda_t^{-1})$, where
\begin{equation}
\begin{aligned}
    \Lambda_t &= \Lambda_{t-1} + X_t^T(\sigma^2I + X_t\Sigma X_t^T)^{-1}X_t\\
    \mu_t &= \Lambda_t^{-1} \left(\Lambda_{t-1}\mu_{t-1} + X_t^T(\sigma^2I + X_t\Sigma X_t^T)^{-1}y_t \right)
\end{aligned}
\end{equation}
Using these recurrences, the posterior after $T$ tasks is a Gaussian with parameters
\begin{equation}
\begin{aligned}
    \Lambda_T &= \Lambda_0 + \sum_{t=1}^T X_t^T(\sigma^2I + X_t\Sigma X_t^T)^{-1}X_t\\
    \mu_T &= \Lambda_T^{-1} \left( \Lambda_{0}\mu_{0} + \sum_{t=1}^T X_t^T(\sigma^2I + X_t\Sigma X_t^T)^{-1}y_t \right)
\end{aligned}
\end{equation}
In a high-dimensional setting, where $N<K$, the above equation is efficient as it handles matrices of size $N \times N$ at cost $O(N^{\omega})$. In a large-sample setting, where $K<N$, using $N \times N$ matrices can be costly. To reduce the computational cost to $O(K^{\omega})$, we apply the Woodbury matrix identity
\begin{equation}
  \left(A+UCV\right)^{-1}
  = A^{-1}-A^{-1}U\left(C^{-1}+VA^{-1}U\right)^{-1}VA^{-1}
\end{equation}
to $(\sigma^2I + X_t\Sigma X_t^T)^{-1}$. This yields new recurrences
\begin{equation}
\begin{aligned}
    \Lambda_t &= \Lambda_{t-1} + \frac{S_t}{\sigma^{2}} - \frac{S_t}{\sigma^{2}} \left(\Sigma^{-1}+ \frac{S_t}{\sigma^{2}}\right)^{-1} \frac{S_t}{\sigma^{2}}\\
    \mu_t &= \Lambda_t^{-1} \left(\Lambda_{0}\mu_{0} + \sum_{t=1}^T \frac{c_t}{\sigma^{2}} - \frac{S_t}{\sigma^{2}} \left(\Sigma^{-1}+ \frac{S_t}{\sigma^{2}}\right)^{-1} \frac{c_t}{\sigma^{2}} \right)
\end{aligned}
\label{eq:linear bandit posterior}
\end{equation}
where
\begin{equation}
\begin{aligned}
    S_t &= X_t^TX_t\\
    c_t &= X_t^Ty_t
\end{aligned}
\end{equation}
Then final posterior parameters turn out to be
\begin{equation}
\begin{aligned}
    \Lambda_T &= \Lambda_0 + \sum_{t=1}^T \frac{S_t}{\sigma^{2}} - \frac{S_t}{\sigma^{2}} \left(\Sigma^{-1}+ \frac{S_t}{\sigma^{2}}\right)^{-1} \frac{S_t}{\sigma^{2}}\\
    \mu_T &= \Lambda_T^{-1} \left( \Lambda_{0}\mu_{0} + \sum_{t=1}^T \frac{c_t}{\sigma^{2}} - \frac{S_t}{\sigma^{2}} \left(\Sigma^{-1}+ \frac{S_t}{\sigma^{2}}\right)^{-1} \frac{c_t}{\sigma^{2}} \right)
\end{aligned}
\end{equation}
This concludes the proof.

\section{Additional Experiments}
\label{sec:additional experiments}

\begin{figure*}[t]
  \centering
  \includegraphics[width=6.6in]{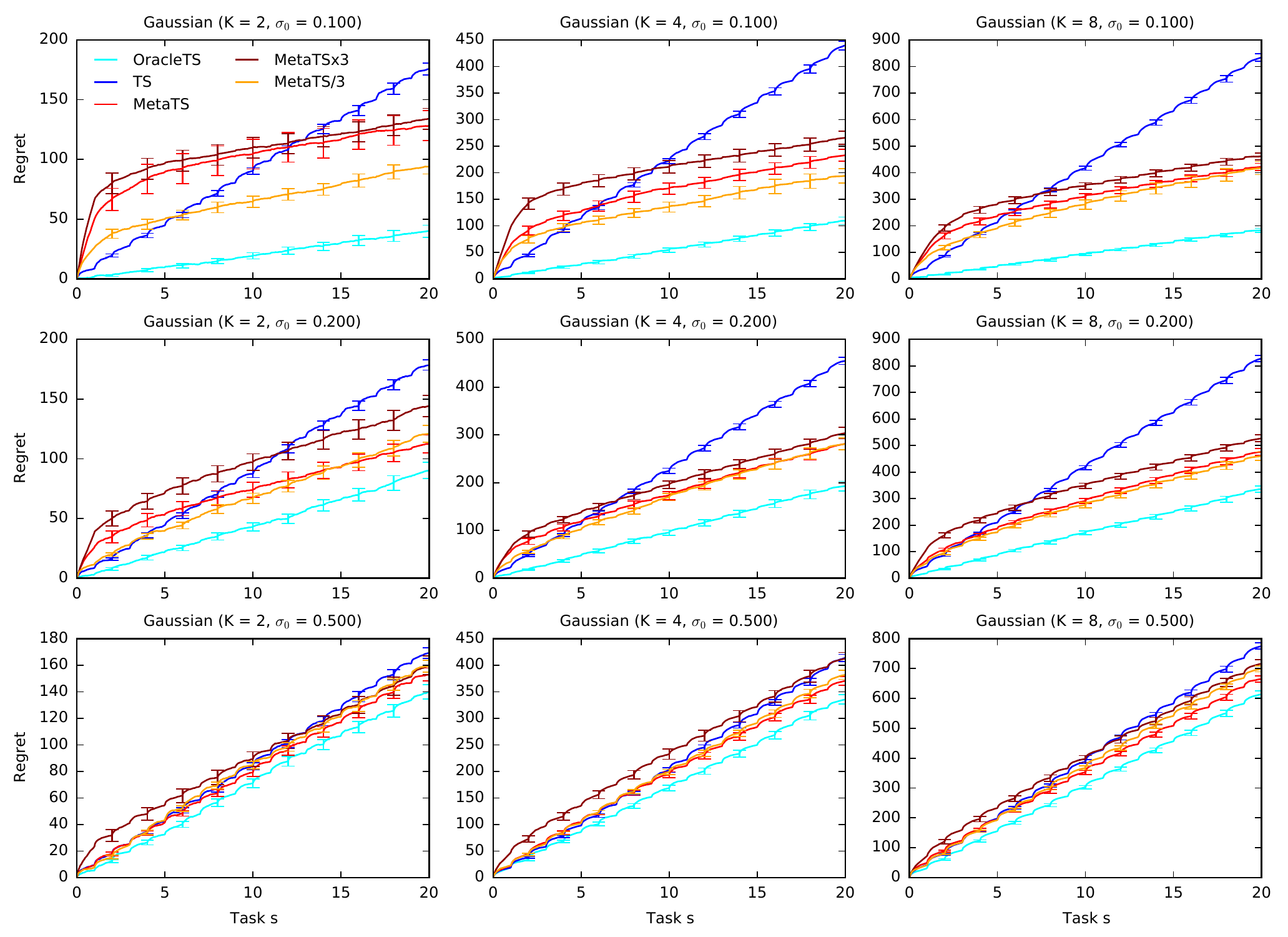}
  \vspace{-0.1in}
  \caption{\metats in $K$-armed Gaussian bandits with a varying number of arms $K$ and instance prior width $\sigma_0$.}
  \label{fig:gaussian}
\end{figure*}

\begin{figure*}[t]
  \centering
  \includegraphics[width=6.6in]{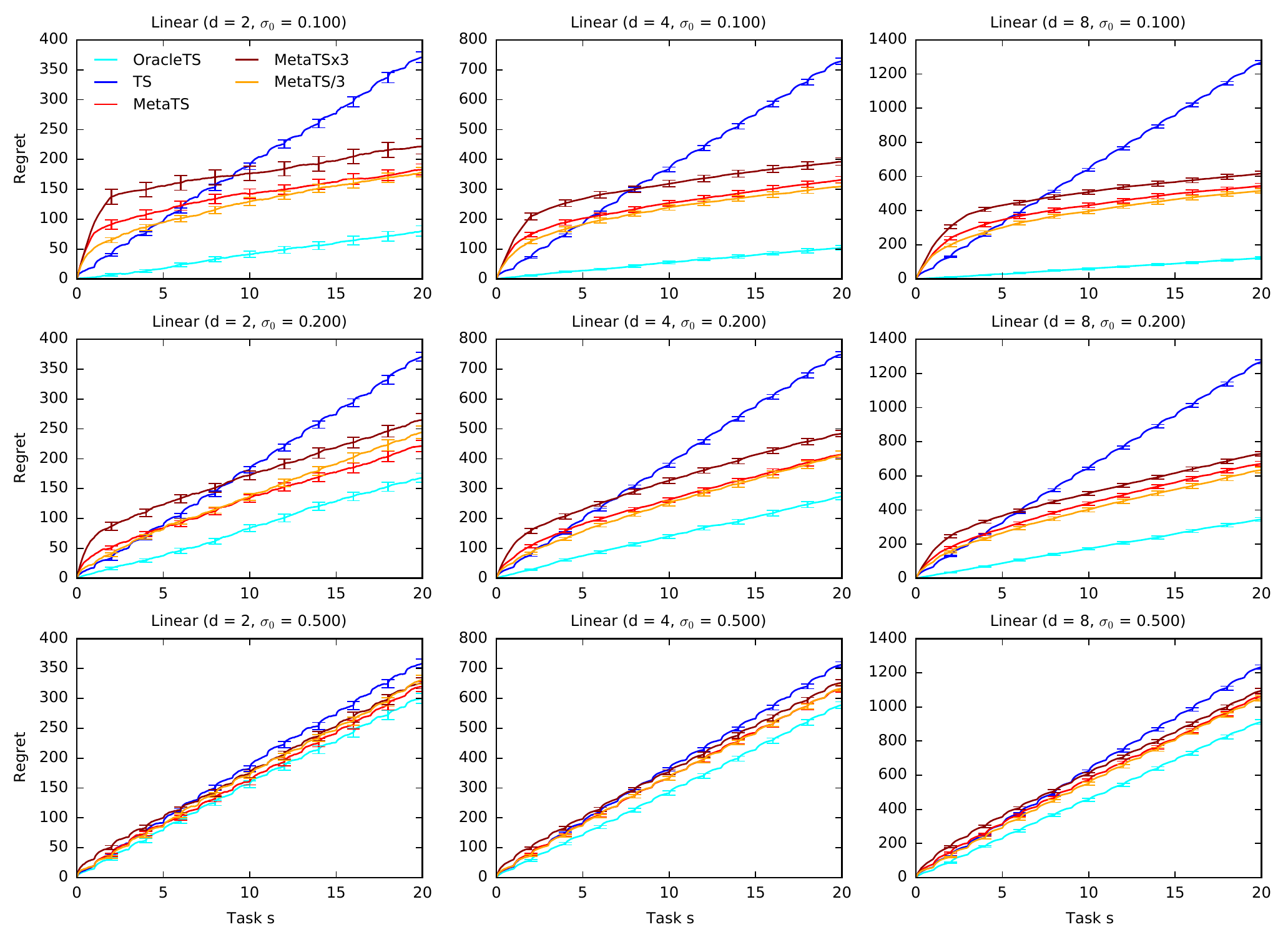}
  \vspace{-0.1in}
  \caption{\metats in $d$-dimensional linear bandits with a varying dimensionality $d$ and instance prior width $\sigma_0$.}
  \label{fig:linear}
\end{figure*}

This section contains additional experiments with \metats.

In \cref{fig:gaussian}, we experiment with $K$-armed Gaussian bandits. The experimental setting is the same as in \cref{sec:experiments}, except that we increase the number of arms $K$ from $2$ to $8$, and the instance prior width $\sigma_0$ from $0.1$ to $0.5$. As $K$ increases, we observe that the benefits of meta-learning do not diminish. As $\sigma_0$ increases, the bandit instances become more uncertain, and this diminished the value of meta-leaning the instance prior.

We also study the robustness of \metats to misspecification. Since the meta-prior $Q$ is the main new concept introduced by our work, we study that one. We experiment with two misspecified variants of \metats: \metatsx and \metatsd. In \metatsx, the meta-prior width is widened to $3 \sigma_q$. This corresponds to an overoptimistic agent, which may commit to an incorrect instance prior. In \metatsd, the meta-prior width is shortened to $\sigma_q / 3$. This corresponds to a conservative agent, which may learn slower than necessary. Our results in \cref{fig:gaussian} show that \metats is robust to misspecification. While the misspecification has an impact on regret, sometimes positive, we do not observe catastrophic failures.

In \cref{fig:linear}, we experiment with $d$-dimensional linear bandits. The experimental setting is the same as in \cref{sec:experiments}, except that we increase $d$ from $2$ to $8$, and the instance prior width $\sigma_0$ from $0.1$ to $0.5$. We observe the same trends as in \cref{fig:gaussian}. This means that our earlier meta-learning conclusions generalize to structured problems.

\end{document}